\documentclass{article}

\usepackage{PRIMEarxiv}

\usepackage[utf8]{inputenc} % allow utf-8 input
\usepackage[T1]{fontenc}    % use 8-bit T1 fonts
\usepackage{hyperref}       % hyperlinks
\usepackage{url}            % simple URL typesetting
\usepackage{booktabs}       % professional-quality tables
\usepackage{amsfonts}       % blackboard math symbols
\usepackage{nicefrac}       % compact symbols for 1/2, etc.
\usepackage{microtype}      % microtypography
\usepackage{xcolor}         % colors

\usepackage{graphicx}  %support the \includegraphics command and options
\usepackage{subfigure} % make it possible to include more than one captioned figure/table in a single float

\usepackage{amsmath,amssymb}
\usepackage{amsthm}
\usepackage{comment}
\usepackage{xr}

\newtheorem{theorem}{Theorem}
\newtheorem{assumption}{Assumption}
\newtheorem{lemma}{lemma}

\newtheorem{remark}{Remark}

\title{Theoretical analysis of deep neural networks for temporally dependent observations
}

\author{
  Mingliang Ma \\
  Department of Statistics\\
  University of Florida\\
  Gainesville, FL 32611 \\
  \texttt{maminglian@ufl.edu} \\
  \And
  Abolfazl Safikhani \\
  Department of Statistics \\
  George Mason University \\
  Fairfax, VA 22030 \\
  \texttt{asafikha@gmu.edu} \\
}

\begin{document}

\maketitle

\begin{abstract}
Deep neural networks are powerful tools to model observations over time with non-linear patterns. Despite the widespread use
of neural networks in such settings, most theoretical developments of deep neural networks are under the assumption of independent observations, and theoretical results for temporally dependent observations are scarce. To bridge this gap, we study theoretical properties of deep neural networks on modeling non-linear time series data. Specifically, non-asymptotic bounds for prediction error of (sparse) feed-forward neural network with ReLU activation function is established under mixing-type assumptions. These assumptions are mild such that they include a wide range of time series models including auto-regressive models. Compared to independent observations, established convergence rates have additional logarithmic factors to compensate for additional complexity due to dependence among data points. The theoretical results are supported via various numerical simulation settings as well as an application to a macroeconomic data set.

% non-linear time series observations. is a promising alternative to the traditional statistic methods in time series forecasting. Despite the widespread use
% of neural networks, the theoretical property is barely known in time series setting. Our work proves the convergence rate of neural network with temporal dependent observations and shows that the convergence rate coincides with the rate with independent observations up to a $\mathrm{log}^4n$ factor. The proof is subsequently used to show the convergence property of neural network in fitting auto-regressive model. The results are supported by our simulation.
\end{abstract}

\section{Introduction}
Neural networks have the ability to model highly complex relationship among data. If input data are observed data in past with future observations as response, neural networks can be utilized to perform time series forecasting. Examples of application of neural networks in forecasting include biotechnology \cite{almeida2002predictive}, finance \cite{odom1990neural}, health sciences \cite{lisboa2002review}, and business \cite{zhang2004neural}, just to name a selected few. Compared to more traditional time series forecasting methods such as ARIMA models \cite{brockwell2002introduction}, neural networks have the ability to detect highly non-linear trend and seasonality. In this work, we analyse the prediction error consistency of (deep) feed-forward neural networks to fit stationary (non-linear) time series models.

% in
% the data, is a promising method for the time series forecasting.
% Compared with some tradition ways like fitting with ARIMA model, neural network has the following advantages: recognize complex data patterns; detect the nonlinear trend and the seasonality without preprocessing. One basic and popular model is the artificial neural network with ReLU activation function. In this work, we theoretically analyse the convergence property of an artificial neural network to fit stationary time series model.

The property of the single hidden layer neural network is well studied in the past few decades. For example, \cite{mccaffrey1994convergence} use single hidden layer neural network with a transformed cosine activation function to show that a sufficiently complex single hidden layer feed-forward network can approximate any member of a specific class of functions to any desired degree of accuracy. Such an approximation property for neural networks with sigmoidal activation function was also analyzed in \cite{barron1993universal,barron1994approximation}. Further, \cite{bach2017breaking} use monotonic homogeneous activation function (general version of ReLU activation function) and show that both the input dimension and number of hidden units have an effect on the convergence rate when using single layer neural networks.

% we approximate the network to a Lipschitz continuous function.

There are many recent works which shed some light on the reasoning behind the good performance of multi-layer (or deep) neural networks. The performance is evaluated via computing mean squared predictive error which is also called the statistical risk. For example, \cite{barron2018approximation} show that the statistical risk of a multi-layer neural network depends on the number of layers and the input dimension of each layer. The problem of applying deep neural network in high-dimensional settings is that the high-dimensional input vector in nonparametric regression leads to a slow convergence rate \cite{bauer2019deep} while the complexity often scales exponentially with the depth or number of units per layer \cite{bartlett2017spectrally, golowich2018size,neyshabur2015norm}. Further, convergence rate for prediction error is also related to the regression function and fast rate can be obtained in special classes of regression functions such as additive and/or composition functions \cite{bauer2019deep,juditsky2009nonparametric,kohler2005adaptive}. To avoid the curse of dimensionality and achieving faster rates, \cite{schmidt2020nonparametric} work under hierarchical composition assumption with a sparse neural network. It is shown that under the independence assumption over input vectors, the estimator utilizing a sparse network achieves nearly optimal convergence rates for prediction error. Finally, \cite{truong2021generalization} use neural network as a classifier for temporally dependent observations which are based on Markov processes. We refer to \cite{fan2019selective} for an overview of deep learning methods.

% \textcolor{red}{Other related work for temporally dependent observations is based on Markov process and use neural network as a classifier, \cite{truong2021generalization}. While our work using mixing assumption on observations and gives an nearly optimal convergence rate for predicting the real value of regression function}

Most of theoretical developments related to prediction error consistency of neural networks are under the assumption that either the input variables are independent or they are independent with the error (noise) term, or both. However, these assumptions are restrictive and may not hold in time series models. To bridge this gap, the main goal of this paper is to establish consistency rates for prediction error of deep feed-forward neural networks for temporally dependent observations. To that end, we focus on multivariate nonparametric regression model with bounded and composite regression functions and apply sparse neural networks with ReLU activation functions for estimation (see more details in Section~\ref{sec:setup}). The modeling framework is similar to \cite{schmidt2020nonparametric} while the independence assumption is relaxed. Specifically, we show that given temporally dependent observations, under certain mixing condition, the statistical risk coincides with the result of independent observations setting with an additional $\mathrm{log}^4(n)$ factor where $n$ is the sample size (Theorem~\ref{th1}). Moreover, utilizing the Wold decomposition,
this result is extended to a general family of stationary time series models in which it is shown that the decay rate of $AR(\infty)$ representation coefficients plays an important role on the consistency rate for prediction error of neural networks (Theorem~\ref{th2}). These results give some insights on the effect of temporal dependence on the performance of neural networks by specifically quantifying the prediction error in such settings. Finally, the prediction performance of neural networks in time series settings is investigated empirically via several simulation scenarios and a real data application (Sections~\ref{sec:sim} and \ref{sec:real}).

% in which input data could be independent and they could be dependent with the error term as well. 

% Inspired by \cite{schmidt2020nonparametric}, we develop the theoretical basis for the forecasting performance of deep neural network in time series setting, where $\mathbf{X}_1, \cdots,\mathbf{X}_n$ are temporal dependent multi-variate observations and $Y_1, \cdots, Y_n$ are from the model

% \begin{equation}\label{eq0}
% Y_t = f(\mathbf{X}_t) + \epsilon_t   
% \end{equation}

% We also give the forecasting performance of deep neural network for a stationary autoregressive process under some idealized assumption. It is an attempt to illustrate the behavior of neural network when it comes to an autoregressive model.
 
% The paper is structured as follows. Section 2 introduces the mathematical model of deep neural networks and establishes a low-dimension structure of a statistics model. Main theorem about the convergence rate of neural network is developed in section 3. In Section 4, we conduct two simulation studies to empirically verify its convergence property. Section 5 provides analysis on real dataset. Section 6 concludes.
 
\textit{Notation}. For two random variables $X$ and $Y$, $X \overset{D}{=} Y$ implies that $X$ and $Y$ have the same distribution. For a matrix $W$, $\Vert W\Vert_{\infty} := \mathrm{max}_{i,j}|W_{ij}|$, $\Vert W\Vert_{0}$ is the number of nonzero entries of $W$, and $\Vert W\Vert_1:=\sum_{i,j}|W_{ij}|$. For a vector $v$, $\vert v\vert_{0}$, $\vert v\vert_{1} $ and $\vert v\vert_{\infty}$ are defined by the same way. We write $\lfloor x\rfloor$  for the smallest integer $\geq x$.
For two sequences $(a_t)_{t\geq 1}$ and $(b_t)_{t\geq 1}$, we write $a_t \lesssim b_t$ if there exists a constant $c \geq 1$ such that $a_t \leq c b_t$ for all $t$. If both $a_t \lesssim b_t$ and $b_t \gtrsim a_t$, we write $a_t \asymp b_t$. Also, $a_t = o (b_t)$ implies $a_t/b_t \rightarrow 0$ as $t \rightarrow \infty$. For a multi-dimension random variable $\mathbf{X}$, $\mathbf{X} \sim N(\mu,\Sigma)$ implies that $X$ has multivariate Gaussian distribution with mean $\mu$ and covariance matrix $\Sigma$. For two functions $f, g$, we use $f\circ g(x)$ to denote $f(g(x))$. Also, $\Vert f\Vert_{\infty}:=\mathrm{sup}_x|f(x)|$ is the sup-norm of $f$, and $(a)_{+} = \max(a,0)$ for $a \in \mathbb{R}$.

\vspace{-0.25cm}

\section{Setup}\label{sec:setup}

\vspace{-0.25cm}

In this section, a brief presentation of feed-forward neural networks is provided in Section~\ref{sec:nn} followed by discussing the modeling framework under consideration in Section~\ref{sec:model}.

\vspace{-0.25cm}

\subsection{Background about multi-layer neural networks}\label{sec:nn}
Multi-layer neural network is composed of three parts: input layer, hidden layers and output layer. We denote the depth of a multi-layer neural network by $L$, which implies that there are $L+1$ layers in total consisting of $L-1$ hidden layers, one input layer and one output layer. We refer to the input layer as the $0$-th layer and the output layer as the $L$-th layer. The multi-layer neural network can be written as

\begin{equation}\label{eq2_1}
f(x): \mathbb{R}^d \longrightarrow \mathbb{R} = W_L \sigma_{\mathbf{v}_{L}}( W_{L-1} \sigma_{\mathbf{v}_{L-1}}(\cdots W_1\sigma_{\mathbf{v}_0}(W_0(x))))),
\end{equation}

% \begin{equation*}
%     f(x): \mathbb{R}^d \longrightarrow \mathbb{R} = W_L \sigma( W_{L-1} \sigma(\cdots W_1\sigma(W_0(x))))),
% \end{equation*}

where $W_i$ is the matrix of weights between $(i-1)$-th and $i$-th layers of the network ($i=1, \ldots, L$) and $\sigma_{\mathbf{v}}$ is a modified ReLU activation function in each layer. Specifically, for the shift parameter $\mathbf{v} = (v_1, \cdots, v_r) \in \mathbb{R}^r$, the activation function $ \sigma_{\mathbf{v}}:\mathbb{R}^r \longrightarrow \mathbb{R}$ is defined as
\begin{equation*}
\sigma_{\mathbf{v}}
\begin{pmatrix}
     x_1   \\
     \vdots \\
     x_r
\end{pmatrix}
=
\begin{pmatrix}
     (x_1-v_1)_{+}   \\
     \vdots \\
     (x_r-v_r)_{+}
\end{pmatrix}.
\end{equation*}
Let $p_i$ denote the number of units in the $i$-th layer (note that $p_0 = d$ and $p_L = 1$). For a fully connected multi-layer neural network, the total number of parameters is $\sum_{i=0}^{L-1} p_{i}p_{i+1}$ that is defined as the size of a  multi-layer neural network \cite{choromanska2015loss}. Similar to \cite{schmidt2020nonparametric}, the entries of all weight matrices $\{W_i\}_{i=1,\cdots,L}$ and shift parameters $\{\mathbf{v}_i\}_{i=1,\cdots,L}$ are assumed to be uniformly bounded. The sparsity level $s$ is defined as the number of all non-zero parameters in $\{W_i\}_{i=1,\cdots,L}$ and $\{\mathbf{v}_i\}_{i=1,\cdots,L}$. We also assume that the value of $|f|$ is bounded by some constant $F$. In summary, we focus on the collection of $s$-sparse multi-layer neural networks with bounded parameters which is denoted by $\mathcal{F}(L,p,s,F)$ and defined as
\begin{equation*}
\begin{split}
&\mathcal{F}(L,\mathbf{p},s,F):= 
\{
f \in \mathcal{F}(L,\mathbf{p}): \sum_{i=0}^{L}\Vert W_i \Vert_0 + |\mathbf{v}_i|_0\leq s , \Vert f\Vert_{\infty}\leq F
\},      \\
& \text{where} \,\,\, \mathcal{F}(L,\mathbf{p}):= 
\{
f \ \text{of form (\ref{eq2_1})}: \mathop{\mathrm{max}}_{i=0,1,\cdots,L} \Vert W_i \Vert_{\infty} \vee |\mathbf{v}_i|_{\infty}\leq 1
\}.
\end{split}
\end{equation*}

This restriction of neural networks to the ones with sparse connections and bounded parameters is common on deep learning (see e.g. \cite{schmidt2020nonparametric} and references therein) since neural networks are typically trained using certain penalization methods and dropouts.

\subsection{Model}\label{sec:model}
The modeling framework considered is similar to \cite{schmidt2020nonparametric} while allowing for temporal dependence among observations. Let $(\epsilon_t)_{t\geq 1}$ be a sequence of independent random variables with $\mathbb{E}[\epsilon_t]=0$. Let $\{\mathbf{X}_t\}_{t\geq 1}$ be a p-dimensional stationary process with $\mathbf{X}_t := (X_{t,1},\cdots,X_{t,p})$. We assume $Y_t$ is generated as
\begin{equation}\label{sect2_eq1}
   Y_t = f_0(\mathbf{X}_t) +\epsilon_t, 
\end{equation}
with a measurable function $f_0:\mathbb{R}^p \longrightarrow \mathbb{R}$.
We assume that the regression function $f_0$ is a composition of several functions, specifically
\begin{equation}\label{model_eq2}
    f_0 = g_q \circ g_{q-1} \circ \cdots \circ g_1 \circ g_0,
\end{equation}
with $g_i:[a_i,b_i]^{d_i}\longrightarrow [a_{i+1},b_{i+1}]^{d_{i+1}}$, where $d_0=p, d_{q+1}=1$. Each $g_i$ has a $d_{i+1}$-dimensional vector output $g_i = (g_{i,1},\cdots,g_{i,d_{i+1}})$. We assume that the multivariate function $g_{i,j}$ depends on at most $t_i$ variables while $t_i$ is far less than $d_i$, i.e. $t_i \ll d_i$. As mentioned in \cite{schmidt2020nonparametric}, for a $\beta$-smooth function $f_0$, the minimax estimation rate for the prediction error is $n^{-2\beta/(2\beta + d_0)}$. Since the dimensionality $d_0$ can be large in applications, the rate can be slow. To mitigate this issue, the sparse structure avoids the effect of input dimension on the convergence rate and improves the rate.

% As an example, consider the function $f_0(x_1, x_2, x_3) = g_{11}(g_{01}(x_1, x_3), g_{02}(x_1, x_2))$ for which
% $d_0 = 3, t_0 = 2, d_1 = t_1 = 2$ and $d_2 = 1$.

% The special form of $f_0$ in eq(\ref{model_eq2}) is come up with by \cite{schmidt2020nonparametric} and has a low-dimension structure due to the assumption that $t_i \ll d_i$. For a $\beta$-smooth function $f_0$, the minimax estimation rate for the prediction error is $n^{-2\beta/(2\beta + d_0)}$. Since the dimensionality $d_0$ can be super large in real situation, the rate can be very slow. While the sparse structure avoids the effect of input dimension on the convergence rate and increases the rate dramatically. 

Let $T$ be a region in $\mathbb{R}^r$. Let $\beta$ and $L$ be two positive
numbers. The H\"older class $\Sigma(\beta,L)$ is defined as the set of $\alpha = \lfloor \beta \rfloor$
times differentiable functions $f : T \longrightarrow \mathbb{R}^r$ whose derivative $\partial^{\alpha}f(\mathbf{x})$ satisfies
\begin{equation*}
\frac{|\partial^{\alpha} f(\mathbf{x})-\partial^{\alpha} f(\mathbf{y})|}{|\mathbf{x}-\mathbf{y}|_{\infty}^{\beta-\lfloor\beta\rfloor}} \leq L,
\end{equation*}

where we used the notation $\partial^{\alpha} = \partial^{\alpha_1}\cdots\partial^{\alpha_r} $ with $\alpha = (\alpha_1,\cdots,\alpha_r)$ and $|\alpha|:=| \alpha |_1 $. Further, we define the ball of $\beta$-H\"older functions with radius $K$ as

\begin{equation*}
\begin{split}
\mathcal{C}_r^{\beta}(D,K) = 
\{  f : 
&D \subset \mathbb{R}^r \longrightarrow \mathbb{R}:\\
&\sum_{\alpha : |\alpha|<\beta} \Vert \partial^{\alpha}f \Vert_{\infty}
+
\sum_{\alpha:|\alpha|=\lfloor \beta \rfloor} \mathop{\mathrm{sup}}_{\mathbf{x},\mathbf{y}\in D}
\frac{|\partial^{\alpha} f(\mathbf{x})-\partial^{\alpha} f(\mathbf{y})|}{|\mathbf{x}-\mathbf{y}|_{\infty}^{\beta-\lfloor\beta\rfloor}} 
\leq 
K
\}.
\end{split}
\end{equation*}

We assume that all functions $g_{ij}$ for $i= 0,\cdots, q$ and $j= 1,\cdots, d_{i+1}$ belong to $\beta_i$-H\"older class $\mathcal{C}_{t_i}^{\beta_{i}}(D_{ij},K)$. From the model (\ref{model_eq2}), we know that $D_{ij}=[a_i,b_i]^{t_i}$. Hence, the class of $f_0$ we focus belongs to
\begin{equation}\label{sect2_eq2}
\begin{split}
\mathcal{G}(q,\mathbf{d},\mathbf{t},\boldsymbol\beta,K):= 
\{
f_0=g_q\circ \cdots\circ g_0 : g_i = (g_{ij})_j:[a_i,b_i]^{d_i}\longrightarrow [a_{i+1},b_{i+1}]^{d_{i+1}},\\
g_{ij}\in \mathcal{C}_{t_i}^{\beta_{i}}([a_i,b_i]^{t_i},K), \ \text{for} \ \text{some} \ |a_i|,|b_i|\leq K
\},
\end{split}
\end{equation}
with $\mathbf{d} := (d_0,\cdots, d_{q+1}), \mathbf{t} := (t_0,\cdots, t_q)$ and $\boldsymbol\beta := (\beta_0, \cdots, \beta_q)$.

\section{Main result}\label{sec:main results}

% \subsection{neural network with temporal dependent observations}

In this section, we present consistency results for prediction error of deep neural networks applied to model~\eqref{sect2_eq1} followed by providing time series model examples satisfying the assumptions in Section~\ref{sec:ar}. First, we need to introduce some notations and state the main assumptions under which the theoretical developments are established. For any estimator $\widehat{f}_n$ in the class $\mathcal{F}(L,p,s,F)$, we define (similar to \cite{schmidt2020nonparametric}) $\Delta_n(\widehat{f}_n,f_0)$ to measure the difference between the expected empirical risk of $\widehat{f}_n$ and the global minimum over all networks in the class $\mathcal{F}(L,p,s,F)$ as
\begin{equation*}%\label{sect3_eq1}
\begin{split}
&\Delta_n(\widehat{f}_n,f_0)\\
&\quad:=
\mathbb{E}_{f_0}\left[\frac{1}{n}\sum_{i=1}^n(Y_i-\widehat{f}_n(\mathbf{X}_i))^2
-
\mathop{\mathrm{inf}}_{f \in \mathcal{F}(L,p,s,F)}\frac1n\sum_{i=1}^n(Y_i-f(\mathbf{X}_i))^2
\right].    
\end{split}
\end{equation*}

% Since training the neural network gives a suboptimal solution, $\widehat{f}_n$ is not the global minimizer the most of time. It follows that $\Delta_n(\widehat{f}_n,f_0)$ is not exact zero. Actually

The quantity $\Delta_n(\widehat{f}_n,f_0)$ plays a pivotal role in consistency properties of neural networks. The performance of $\widehat{f}_n$ is evaluated by the prediction error defined as
\begin{equation}\label{sect3_eq2}
R(\widehat{f}_n,f_0) := \mathbb{E}_{f_0}\left[ (\widehat{f}_n(\mathbf{X})-f_0(\mathbf{X}))^2 \right],
\end{equation}

with $\mathbf{X} \overset{D}{=} \mathbf{X}_t $ and $\mathbf{X}$ is independent with $\{\mathbf{X}_{t}\}_{t\geq 0}$. Recall from Section~\ref{sec:model}, the regression function $f_0$ is in the class $\mathcal{G}(q,\mathbf{d},\mathbf{t},\boldsymbol\beta,K)$. To simplify notations, we define $\beta_i^{*}:= \beta_i \prod_{\ell = i+1}^q (\beta_{\ell}  \wedge 1), \phi_n := \mathop{\mathrm{max}}_{i=0,\cdots,q} n^{-\frac{2\beta_i^{*}}{2\beta_i^{*}+t_i}}$. The following assumptions are needed to present the first theorem.

% \[\beta_i^{*}:= \beta_i \prod_{\ell = i+1}^q (\beta_{\ell}  \wedge 1), \quad \phi_n := \mathop{\mathrm{max}}_{i=0,\cdots,q} n^{-\frac{2\beta_i^{*}}{2\beta_i^{*}+t_i}}.\] The 

\begin{assumption}\label{assum1} For all $i = 1, 2, \ldots$,
$\mathbb{E}[\epsilon_i] =0$, $\mathbb{E}[\epsilon_i^2] =\sigma^2$, and there exists some positive constant $c$ such that $\mathbb{E}[|\epsilon_i|^m] \leq \sigma^2m!c^{m-2}
, m = 3, 4, \cdots$. 
\end{assumption}
\begin{assumption}\label{assum2}
$\{\mathbf{X}_t\}$ is a strictly stationary and exponentially $\alpha$-mixing process. Recall that the $\alpha$-mixing coefficient of a stationary process $\{\mathbf{X}_t\}$ is define as
\begin{equation*}
\alpha(s) = \mathop{\mathrm{sup}}\{ |\mathbb{P}(A\cap B)- \mathbb{P}(A)\mathbb{P}(B)|: -\infty< t <\infty, A\in \sigma(\mathbf{X}_t^{-}), B\in \sigma(\mathbf{X}_{t+s}^{+})  \},
\end{equation*}

where $\mathbf{X}_t^{-}$ consists of the entire past of the process including $\mathbf{X}_t$, and $\mathbf{X}_t^{+}$ consists of its entire future. The process $\{\mathbf{X}_t\}$ is said to be exponentially $\alpha$-mixing if there exists some constant $\tilde{c}>0$ such that $\mathrm{log}(\alpha(t))\leq -\tilde{c}t, t\geq 1$.
% \begin{equation*}
%     \begin{split}
% \mathrm{log}(\alpha(t))\leq -\tilde{c}t, \quad \quad t\geq 1.
%     \end{split}
% \end{equation*}
\end{assumption}
\begin{assumption}\label{assum3}
The error term $\epsilon_t$ is independent with $\{X_s, s\leq t\}.$
\end{assumption}

Assumption 1 is known as the Bernstein condition and implies that $\epsilon_t$ is a sub-exponential variable. This assumption is often used when we cannot assume $\epsilon_t$ is bounded. Assumption 2 is to control the dependence among input variables and holds for a wide range of time series models, see e.g. auto-regressive models in  \cite{doukhan2012mixing}. Assumption 3 controls dependence between input variables and error terms. A more stringent condition is to assume the whole error process $\{ \epsilon_t \}_{t \geq 0}$ is independent with $\{X_s\}_{s \geq 0}$. However, this assumption is restrictive since it excludes auto-regressive model which is an important family of time series models. To avoid this, we only assume the current error term is independent of current and past input variables. Further, this assumption ensures that $\sum_t\epsilon_t f_0(\mathbf{X}_t)$ is a martingale which helps in verifying certain concentration inequalities needed in the proof of main results. All three assumptions are common in non-linear time series analysis, see e.g. \cite{davis2020modeling}. Now, we are ready to state the main result.

% a reasonable assumption for a nonparametric univariate time series. By setting $Y_t=X_{t+1}$ and $\mathbf{X}_t = \{X_{t}, X_{t-1},\cdots,X_{t-p}\}$in model (\ref{sect2_eq1}), $\mathbf{X}_t$
% forms a stationary geometrically ergodic p-th order
% Markov chain(\cite{an1996geometrical}, Theorem 3.1) when $\epsilon_t$ has a continuous density, which is known as a sufficient condition for  exponentially $\alpha$-mixing(e.g. \cite{doukhan2012mixing}, p.89). 

% ensures that $\sum_t\epsilon_t f_0(\mathbf{X}_t)$ is a martingale and is satisfied in time series model. Examples of process which satisfies assumption 3 are: (1) the case where $\{\epsilon_t\}_{t\in \mathbb{Z}}$ and $\{X_t\}_{t\in \mathbb{Z}}$ are two independent processes and $\epsilon_t$ is considered as a noise term in model (\ref{sect2_eq1}) ;(2) augmented AR(p) model, $X_{t}=f(X_{t-1},\cdots,X_{t-p})+g(\mathbf{Z}_t)+\epsilon_t$. When $g(\cdot)=0$, the augmented AR(p) model transforms to the general AR(p) model.

% \begin{itemize}
%     \item[(\romannumeral1)] $F\geq \mathrm{max}(K,1),$
%     \item[(\romannumeral2)] $\sum_{i=0}^{q}\mathrm{log}_2(4t_i\vee 4\beta_i)\mathrm{log}_2n\leq L \lesssim n\phi_n,$
%     \item[(\romannumeral3)] $n\phi_n \lesssim \mathrm{min}_{i=1,\cdots,L}p_i,$
%     \item[(\romannumeral4)] $s \asymp n\phi_n\mathrm{log}n.$
% \end{itemize}

\begin{theorem}\label{th1}
Consider the d-variate nonparametric regression model~\eqref{sect2_eq1} for a composite regression function \eqref{model_eq2} in the class $\mathcal{G}(q,\mathbf{d},\mathbf{t},\boldsymbol\beta,K)$. Suppose Assumptions~1-3 hold. Let $\widehat{f}_n$ be an estimator taking values in the network class $\mathcal{F}(L,(p_i)_{i=0,\cdots,L+1},s,F)$ satisfying (i) $F\geq \mathrm{max}(K,1)$; (ii) $\sum_{i=0}^{q}\mathrm{log}_2(4t_i\vee 4\beta_i)\mathrm{log}_2n\leq L \lesssim n\phi_n$; (iii) $n\phi_n \lesssim \mathrm{min}_{i=1,\cdots,L}p_i$; and (iv) $s \asymp n\phi_n\mathrm{log}n$. Then, there exist positive constants $C, C^{'}$ depending only on $q, \mathbf{d}, \mathbf{t}, \boldsymbol\beta, F$, such that if $\Delta_n(\widehat{f},f_0)\leq C\phi_n L\mathrm{log}^6n$, then
\begin{equation}\label{th1_eq1}
\begin{split}
R(\widehat{f}_n,f_0)\leq C^{'}\phi_n L\mathrm{log}^6n,
\end{split}
\end{equation}
and if $ \Delta_n(\widehat{f},f_0)\geq C\phi_n L\mathrm{log}^6n$, then
\begin{equation}\label{th1_eq2}
\begin{split}
\frac{1}{C^{'}}\Delta_n(\widehat{f},f_0)\leq R(\widehat{f}_n,f_0) \leq C^{'}\Delta_n(\widehat{f}_n,f_0).
\end{split}
\end{equation}
\end{theorem}

Based on Theorem~\ref{th1}, the prediction error defined in \eqref{sect3_eq2} is controlled by $\phi_nL\mathrm{log}^6n$. From condition (ii), $L$ is at least of the order of $\mathrm{log}_2n$. Thus, the rate in Theorem~1 becomes $\phi_n \mathrm{log}^7n$. This rate for the case of independent observations is of order $\phi_n \mathrm{log}^3n$ based on Theorem~1 in \cite{schmidt2020nonparametric}. Compared to latter, our rate has an extra $\mathrm{log}^4 n$ factor, which compensates for additional complexity in verifying the prediction error consistency in the presence of temporal dependence among input variables. Further, note that for a fully connected neural network, the number of parameters is $\sum_{i=1}^{L-1}p_ip_{i+1} \gtrsim n^2\phi_n^2L$. We can see that this number is greater than the sparsity level $s$ which is of order $n\phi_n\mathrm{log}n$ based on condition (iv) in Theorem~\ref{th1}. Thus, it can be seen that condition (iv) restricts the neural network class to the ones with sparse architecture. In other words, at least $\sum_{i=1}^{L-1}p_ip_{i+1} - s$ units of the neural network is completely inactive.

% the optimal rate we can get from theorem $\ref{th1}$ becomes $\phi_n \mathrm{log}^7n$. Compared with the convergence rate in independent observations setting (the rate is $\phi_n \mathrm{log}^3n$ given by theorem 1, 
% \cite{schmidt2020nonparametric}), our result has an extra $\mathrm{log}^4 n$ factor. The extra factor we obtain in the convergence rate is likely an artifact of the proof. We empirically show that the performance of neural network with temporal dependent observations and with independent observations are similar.

\subsection{Time series model examples}\label{sec:ar}
In this section, we introduce some (well-known) examples of time series models that satisfy the assumptions of Theorem~\ref{th1}. The first example is to let $\{\epsilon_t\}_{t\in \mathbb{Z}}$ and $\{\mathbf{X}_t\}_{t\in \mathbb{Z}}$ be two independent processes. This independence assumption implies that the input variables $\mathbf{X}_t$ are exogenous, thus Assumption~\ref{assum3} is automatically satisfied. Further, assume $\epsilon_t$ satisfy the moment conditions in Assumption~\ref{assum1} (for example, they have normal distribution) and $\mathbf{X}_t$ is a stationary and geometrically $\alpha$-mixing process. There are many examples of such processes including certain finite-order auto-regressive processes, see e.g. \cite{fan2008nonlinear,doukhan2012mixing}. The second example is to consider non-linear auto-regressive models, i.e. assume

\begin{equation}\label{eq:non-linear ar}
    X_t = g(X_{t-1}, \cdots, X_{t-d}) + \epsilon_t.
\end{equation}

% \textbf{Examples of time series model}. First example is that $\{\epsilon_t\}_{t\in \mathbb{Z}}$ and $\{\mathbf{X}_t\}_{t\in \mathbb{Z}}$ are two independent processes. $\epsilon_t$ satisfies assumption \ref{assum1}, \ref{assum3} and is considered as a noise term. In this case, we can apply theorem \ref{th1} with any time series $\mathbf{X}_t$ satisfying the exponentially $\alpha$-mixing assumption. Example 2 is the following auto-regressive time series model

This is a special case of model \eqref{sect2_eq1} by setting $\mathbf{X}_t=(X_{t-d},\cdots,X_{t-1})$ and $Y_t = X_t$. Assuming $\epsilon_t$'s are i.i.d. random variables with positive density in the real line and the boundedness of the function $g$, it can be shown that there exists a stationary solution to equation~\eqref{eq:non-linear ar} while the solution is exponentially $\alpha$-mixing as well \cite{an1996geometrical,doukhan2012mixing}. In both examples, assumptions of Theorem~\ref{th1} are satisfied, thus the results of this Theorem are applicable.

% In this case, $\mathbf{X}_t$ forms a stationary geometrically ergodic p-th order Markov chain(\cite{an1996geometrical}, Theorem 3.1) when the density of $\epsilon_t$ is continuous and positive almost everywhere, which is known as a sufficient condition for  exponentially $\alpha$-mixing(e.g. \cite{doukhan2012mixing}, p.89). We therefore can use the result of theorem \ref{th1} when fitting an auto-regressive model by neural networks.

% \textbf{Analysis of autoregressive model}. 

Now, we consider a more general time series model. Recall that by the well-known Wold representation, every purely nondeterministic stationary and zero-mean stochastic process ${X_t}$ can be expressed as $X_t = \sum_{i=0}^{\infty} a_i\epsilon_{t-i}$ where $\epsilon_t$ is a mean-zero white noise. Further, if ${X_t}$ has a non-vanishing spectral density and absolute summable auto-regressive coefficients, i.e. $\sum_{i=1}^{\infty} |\phi_i| < \infty$, it has the $\text{AR}(\infty)$ representation $X_t = \sum_{i=1}^{\infty} \phi_iX_{t-i} + \epsilon_t$ (see e.g. \cite{wang2021consistent}). Motivated by this discussion, we consider a general family of times series models satisfying

\begin{equation}\label{sect3_eq3}
X_t = \sum_{i=1}^{\infty} \phi_i X_{t-i} + \epsilon_t,
\end{equation}

where $\epsilon_t$'s are i.i.d. errors. Independence among $\epsilon_t$'s is a strong assumption compared to only assuming that they are uncorrelated, but this is required for our theoretical analysis. The interesting fact about model~\eqref{sect3_eq3} is that it is a linear model. However, since there are infinite covariates in this AR($\infty$) representation, training neural networks directly is impossible. The common solution is to truncate the covariates and only consider the first few, i.e. approximate model~\eqref{sect3_eq3} by an AR(d) model for some $d$. This approximation can successfully estimate second order structures of the original model (i.e. spectral density or auto-correlation function) if $d$ is selected carefully and under certain assumptions on the AR coefficients $\phi_i$'s (see e.g. \cite{wang2021consistent}). Therefore, we follow this path and fit a neural network to the $d$-dimensional input variables $\left( X_{t-1}, \ldots, X_{t-d} \right)$ with a proper choice of $d$ while keeping in mind that the true regression function is in fact $f_0(\mathbf{X}_t) = \sum_{i=1}^{\infty} \phi_i X_{t-i}$. To establish the prediction error consistency of neural networks on truncated input variables, we need two additional assumptions.

% infinity AR model to be a AR(d) model. \cite{wang2021consistent} shows that autoregressive spectral estimator and the associated AR-based autocovariance matrix estimator are consistent if we fit (\ref{sect3_eq3}) with AR(d) model. 

% Therefore, we have reason to believe that the estimator of a neural network $\widehat{f}_n$ still has a good performance with a proper choice of $d$. Recall from eq(\ref{sect3_eq2}), the performance of $\widehat{f}_n$ is evaluated by $R(\widehat{f}_n,f_0) := \mathbb{E}_{f_0}\left[ (\widehat{f}_n(\mathbf{X}_{i})-f_0(\mathbf{X}_{i}))^2 \right]$. For $AR(\infty)$ model, $f_0(\mathbf{X}_t) = \sum_{i=1}^{\infty} \phi_i X_{t-i}$. We still use sparse and bounded estimator $\widehat{f}_n$ in the class $\mathcal{F}(L,p,s,F)$ to fit model (\ref{sect3_eq3}). To establish the theorem, we need the following assumption, 

\begin{assumption}\label{assum4}
There exists $ \alpha > 0, M>0$ such that $\sum_{i=1}^{\infty}(1+i)^{\alpha}|\phi_i| \leq M < \infty.$ 
\end{assumption}
\begin{assumption}\label{assum5}
For some constant $K>0$, $|X_t| \leq K$ for all $t \geq 0$.
\end{assumption}

% case of autoregressive model, AR($\infty$) model, and analysis the performance of neural network. $AR(\infty)$ model  has the following linear infinite order autoregressive representation,
% \begin{equation}\label{sect3_eq3}
% X_t = \sum_{i=1}^{\infty} \phi_i X_{t-i} + \epsilon_t 
% \end{equation}
% with i.i.d. errors $\epsilon_t$. The infinity form of Model (\ref{sect3_eq3}) is related to Wold representation with some idealization.

% \[
% X_t = \sum_{i}^{\infty} a_i\epsilon_{t-i}, \quad  \epsilon_t \sim WN (0,\sigma^2)
% \]
% $WN (0, \sigma^2)$ means white noise with zero mean and variance $\sigma^2$. Furthermore, $X_t$ has a unique autoregressive coefficients $(\phi_i, i=0,1,\cdots)$, if it has nonvanishing spectral density, Under the absolute summable assumption, that is $\sum_{i}|\phi_i| < \infty$, Wold-type representation of the underlying process is given by
% \[
% X_t = \sum_{i}^{\infty} \phi_iX_{t-i} + \epsilon_t, \quad  \epsilon_t \sim WN (0,\sigma^2)
% \] For Wold-type representation, $(\epsilon_t,\ t\in\mathbb{Z})$ is uncorrelated with a common variance. Model (\ref{sect3_eq3}) make a strong assumption on $(\epsilon_t,\ t\in\mathbb{Z})$ that are identical independent distributed.

Assumption~\ref{assum4} controls the decay rate of the $\text{AR}(\infty)$ coefficients in the true model and $\alpha$ can be treated as a decreasing rate of $\phi_i$'s. This assumption is needed to compensate for approximating a general time series of form \eqref{sect3_eq3} with a finite lag AR process. Further, it plays an important role in restricting the first derivative of $f_0$, which corresponds to the $\beta_i$-smoothness assumption on $g_{ij}$ in (\ref{sect2_eq2}). Note that Assumption~\ref{assum4} is satisfied if the spectral density function is strictly positive and continuous, and the auto-covariance function of $X_t$ has some bounded property \cite{kreiss2011range}. Moreover, since the model is linear (i.e. the regression function is unbounded), Assumption~\ref{assum5} becomes necessary to make $f_0(\mathbf{X}_t)$ bounded, a property needed for Theorem~\ref{th1} as well.

%  . Since we use bounded neural networks to make a prediction, assumption \ref{assum5} is necessary.

% Next, we'll show the main result for estimating the $AR(\infty)$ model by a neural network. The converging rate is related to order $d$ of the autoregressive fit we choose. Training data $X_t$ on the past $d$ observations $(X_{t-1},\cdots,X_{t-d} )$ with $d \asymp n^{\frac{1}{\alpha + 1}}$, we have the following theorem.

\begin{theorem}\label{th2}
Consider model (\ref{sect3_eq3}) with $f_0(\mathbf{X}_t) = \sum_{i=1}^{\infty} \phi_i X_{t-i}$. Let $\widehat{f}_n$ be an estimator taking values in the network class $\mathcal{F}(L,(p_i)_{i=0,\cdots,L+1},s,F)$ satisfying (\romannumeral1) $F\geq KM$; (\romannumeral2) $L \geq 4$; (\romannumeral3) $s \asymp Ld $, and (\romannumeral4) $ d \lesssim \mathrm{min}_{i=1,\cdots,L} p_i$. Assume that $d \asymp n^{\frac{1}{\alpha + 1}}$. Under the Assumptions~\ref{assum1}-\ref{assum5}, there exist positive constants $C, C^{'}$ such that
if $\Delta_n(\widehat{f},f_0)\leq C n^{-\frac{\alpha}{\alpha+1}} L\mathrm{log}^5n$, then
\begin{equation}\label{th2_stat_eq1}
\begin{split}
R(\widehat{f}_n,f_0)\leq C^{'}n^{-\frac{\alpha}{\alpha+1}} L\mathrm{log}^5n,
\end{split}
\end{equation}
and if $\Delta_n(\widehat{f},f_0) > C n^{-\frac{\alpha}{\alpha+1}} L\mathrm{log}^5n$, then
\begin{equation}\label{th2_stat_eq2}
\begin{split}
\frac{1}{C^{'}}\Delta_n(\widehat{f},f_0)\leq R(\widehat{f}_n,f_0) \leq C^{'}\Delta_n(\widehat{f},f_0).
\end{split}
\end{equation}
\end{theorem}

Based on Theorem~\ref{th2}, the best convergence rate of $R(\widehat{f}_n,f_0)$ for model~\eqref{sect3_eq3} is $n^{-\frac{\alpha}{\alpha+1}}\mathrm{log}^5n$. Compared with the convergence rate in Theorem~\ref{th1}, we can see that the convergence rate in model~\eqref{sect3_eq3} depends on the decreasing rate of coefficients $\phi_i$'s instead of the smoothness of regression function. Also the logarithmic factor changes from $\mathrm{log}^7n$ to $\mathrm{log}^5n$. Such a subtle decrease is due to the fact that since the truncated model is linear, a shallower and sparser neural network can be used in the proof of Theorem~\ref{th2} (as seen from conditions (\romannumeral2) and (\romannumeral3) in the statement of Theorem~\ref{th2}). Also, note that for a simple linear AR(d) model where $\phi_j = 0$ for $j>d$, Assumption~\ref{assum4} is satisfied for any large $\alpha$. Thus, in this case, the best convergence rate of prediction error becomes $n^{-1}\mathrm{log}^5n$.

\begin{remark}
To fit the neural network to a truncated model~\eqref{sect3_eq3}, shallow neural networks are sufficient. In fact, $L = 4$ is enough based on Theorem~\ref{th2}. This is because model~\eqref{sect3_eq3} has a simple linear structure compared to composition functions \eqref{model_eq2}.
\end{remark}

\begin{remark}
For a general selection of input vector $d$, the result of Theorem~\ref{th2} becomes $R(\widehat{f},f_0)\leq
C\frac{1}{d^{\alpha}} + C^{'}\frac{dL\mathrm{log}^5n}{n}+ 4\Delta_n(\widehat{f},f_0)$. With the choice of $d\asymp n^{\frac{1}{\alpha + 1}}$, we balance the first two terms and establish a proper upper bound for $R(\widehat{f}_n,f_0)$. Thus, this selection can be regarded as the ``optimal selection" of lag $d$ when applying neural networks for estimation in model~\eqref{sect3_eq3}.

% The dimension of input vector $d$ is not a fixed value. The choice of $d$ derives from eq(44) in supplement file, $R(\widehat{f},f_0)\leq
% C\frac{1}{d^{\alpha}} + C^{'}\frac{dL\mathrm{log}^5n}{n}+ 4\Delta_n(\widehat{f},f_0)$. With the choice $d\asymp n^{\frac{1}{\alpha + 1}}$, we balance the first term and the second term and get the restrict upper bound for $R(\widehat{f}_n,f_0)$. It implies that we need to consider more lag values of $X_t$ to make a better prediction as the sample size increases.
\end{remark}

\section{Simulation experiments}\label{sec:sim}

In this section, we conduct two simulation settings to illustrate the performance of neural networks applied to temporally dependent data (see also Section~C in the supplementary materials for numerical comparisons between feed-forward neural networks and LSTM). In the first simulation in Section~\ref{sec:sim:1}, we aim to compare the convergence rate in dependent observations setting with the rate in independent observations setting. The convergence rate is explained as how fast the mean squared predictive error decreases as the sample size grows. In the second simulation (Section~\ref{sec:sim:2}), motivated by time series model examples introduced in Section~\ref{sec:ar}, we use neural networks to fit linear and non-linear auto-regressive models and compare the performance with the result of linear regression method (least squares method). All simulations are repeated 200 times.

To train the neural network, we split the data into three parts: training set  $\mathcal{T}_1 $, validation set $\mathcal{T}_2$, and testing set $\mathcal{T}_3$. We set $\mathcal{T}_1 = \{1,\cdots,n/2\}$, $\mathcal{T}_2 = \{n/2,\cdots,3n/4\}$, and $\mathcal{T}_3 = \{3n/4,\cdots,n\}$. We penalize parameters of weight matrix of the neural network and use mean square error as our loss. To be more specific, the loss function is defined as
\begin{equation}\label{sect4_eq1}
    \text{loss} := \sum_{t \in \mathcal{T}_1}(Y_t - \widehat{Y}_t)^2/(n/2) + \lambda \sum_{i=1}^L\|W_i\|_1,
\end{equation}
where $\lambda$ is the sparsity tuning parameter. The value of $\lambda$ does not have significant effect on simulation results. In this section, we set $\lambda = 0.1$. We apply gradient descent method to update parameters of neural network and stop the iteration when the neural network gives the minimum mean square error for $\mathcal{T}_2$. The prediction error $R(\widehat{f}_n, f_0)$ defined in \eqref{sect3_eq2} is empirically estimated by $\widehat{R} = \sum_{t \in \mathcal{T}_3}(f_0(\mathbf{X}_t) - \widehat{Y}_t)^2/(n/4)$. We use $\widehat{R}$ to evaluate the performance of the trained neural network. 

\subsection{Dependent vs. independent observations}\label{sec:sim:1}
In this experiment, we consider a nonlinear additive model $f_0(\mathbf{X}_t) = 2\mathop{\sum}_{i=1}^{4} \cos(X_{t,i})$. Thus, $Y_t$ is from the following model
\begin{equation}\label{eq4_1}
 Y_t = 2\mathop{\sum}_{i=1}^{4} \cos(X_{t,i}) + \eta_t  ,\quad t=1,2,\cdots, n,
\end{equation}

with $\eta_t$'s as independent standard normal random variables, and $\mathbf{X}_t:=(X_{t,1},\cdots, X_{t,4}) $ is generated from an AR(4) model. Specifically,
\begin{equation*}
\mathbf{X}_t = 
\begin{bmatrix}
0  & \rho & 0 & 0 \\
0  & 0 & \rho & 0 \\
0  & 0 & 0  &\rho \\
0  & 0 & 0  & 0
\end{bmatrix}
\mathbf{X}_{t-1} + \boldsymbol\epsilon_t, \quad \boldsymbol\epsilon_t\sim N(\mathbf{0}, 
\mathbf{I}_4),
\end{equation*}

where $\rho$ takes two value, $\rho = 0.2,\ 0.6$. The sample size is $n = 100$, $400$, $1600$, $6400$. A three layer neural network is selected to fit model (\ref{eq4_1}) while ReLU is used as the activation function. The network requires a 4- dimensional input vector and has 20 units in each hidden layer. To make a fair comparison with the performance of neural network in independent observations setting, we still use model \eqref{eq4_1} but with independent observations, i.e. we generate $\mathbf{X}^{'}_1,\cdots,\mathbf{X}^{'}_n\sim N(\mathbf{0},\boldsymbol\Sigma)$ i.i.d. and $\boldsymbol\Sigma$ derives from $\mathbf{X}^{'}_1 \overset{D}{=} \mathbf{X}_1 $.

\begin{figure}[!ht]
\centering
\subfigure[$\rho = 0.2$]{
\begin{minipage}[t]{0.45\linewidth}
\centering
\includegraphics[width=2.2in]{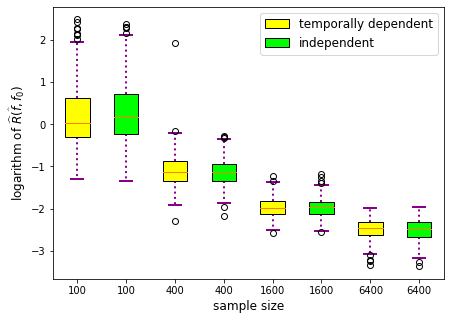}
%\caption{fig1}
\end{minipage}%
}%
\subfigure[$\rho = 0.6$]{
\begin{minipage}[t]{0.45\linewidth}
\centering
\includegraphics[width=2.2in]{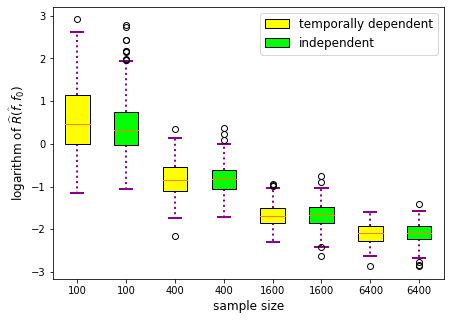}
%\caption{fig2}
\end{minipage}%
}%
\centering
\caption{Box plots of logarithm of the mean square error on the testing set as a function of sample size.}
\label{fig1}
\end{figure}

Figure \ref{fig1} illustrates $\log \left( \widehat{R} \right)$ against various sample sizes for both temporally dependent case and independent case by running 200 replications for each. As can be seen, the prediction performance of neural network improves as sample size increases. Further, we observe that the prediction error for the neural network with temporally dependent data and independent one are similar. This may imply that the additional logarithmic terms appearing in Theorem~\ref{th1} might be an artifact of the proof.

\subsection{Auto-regressive examples}\label{sec:sim:2}
In this experiment, we test the performance of the neural network and compare its result with simple linear regression (least squares method) for several linear and non-linear auto-regressive models. Specifically, we consider the following four models: (1) $X_{t} = 0.6 X_{t-1} + \epsilon_t$; (2) $X_{t} = 0.6 X_{t-1} - 0.4X_{t-2} + 0.2X_{t-3}  + \epsilon_t$; (3) $X_t = 0.5 \sqrt{|X_{t-1}|} + \epsilon_t $; (4) $X_t = 0.5|X_{t-1}| + \epsilon_t$. The error term is generated as $\epsilon_1, \cdots, \epsilon_n\sim N(0,1)$ i.i.d. in all models. The sample size is  $n = 100$, $400$, $1600$, $6400$.

We again use a four layer neural network with 20 units in each hidden layer. This time, we have no prior knowledge on the input dimension of the neural network, since lags of these four time series models are unknown. To determine the input dimension (lag of time series), we apply linear regression with AIC criterion. The AIC is defined as  $\text{AIC} = n \mathrm{log}(SSE) + 2d $, in which $d$ is the input dimension and $SSE$ is the summation of squared errors. 

\begin{figure}[!ht]
\centering
\subfigure[model (1)]{
\begin{minipage}[t]{0.45\linewidth}
\centering
\includegraphics[width=2.2in]{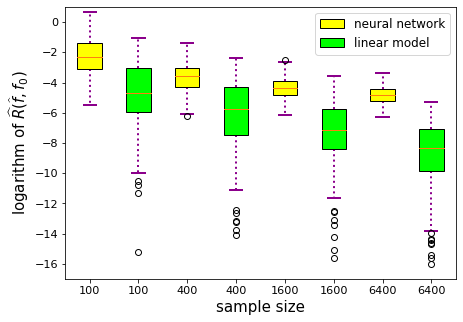}
%\caption{fig1}
\end{minipage}%
}%
\subfigure[model (2)]{
\begin{minipage}[t]{0.45\linewidth}
\centering
\includegraphics[width=2.2in]{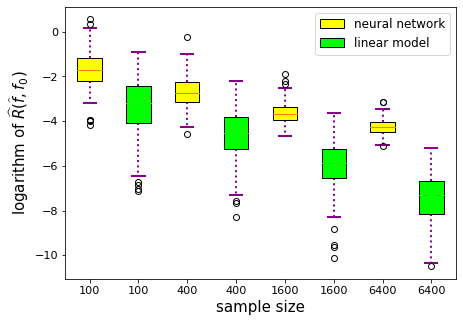}
%\caption{fig2}
\end{minipage}%
}%
\\
\centering
\subfigure[model (3)]{
\begin{minipage}[t]{0.45\linewidth}
\centering
\includegraphics[width=2.2in]{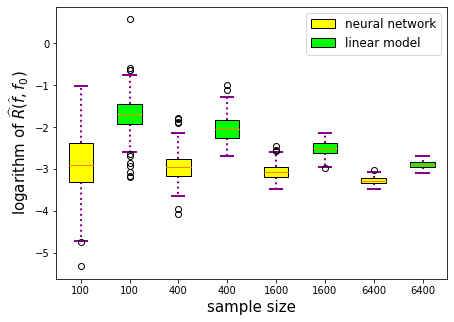}
%\caption{fig1}
\end{minipage}%
}%
\subfigure[model (4)]{
\begin{minipage}[t]{0.45\linewidth}
\centering
\includegraphics[width=2.2in]{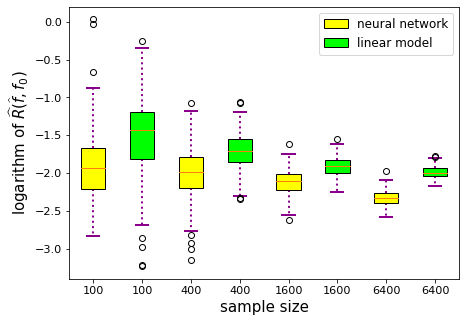}
%\caption{fig2}
\end{minipage}%
}%
\centering
\caption{Box plots of logarithm of the mean square error on the testing set as a function of sample size.}
\label{fig2}
\end{figure}

Figure \ref{fig2} displays the logarithm of $\widehat{R}(\widehat{f},f_0)$ as sample size increases. Linear regression method has a fast convergence rate in AR(1) and AR(3) models, i.e. models (1) and (2). For non-linear auto-regressive models (models (3) and (4)), the convergence rate of linear regression method becomes slow while neural network outperforms the linear method in these cases. In summary, neural network has the advantage to deal with estimation in both linear and non-linear auto-regressive models.

\section{Real data analysis}\label{sec:real}

In this section, we apply neural network on a macroeconomic data\footnote{link : https://www.sydneyludvigson.com/data-and-appendixes} to predict monthly inflation rate. The data consist of 132 monthly macroeconomic variables from January 1960 to December 2011, a total of 624 time periods. The inflation rate is measured by the percentage changes of the Consumer Price Index (CPI). To be more specific, we are interested in forecasting
\begin{equation}\label{sect5_eq1}
    \pi_{t} = 1200 \times  \mathrm{log}\left(\frac{\text{CPI}_{t+1}}{\text{CPI}_t}\right).
\end{equation}
Similar to \cite{medeiros2016l1,wu2016performance}, we use the AR(4) model as the benchmark model, i.e. $\pi_{t+1} = \alpha_0 + \sum_{i=0}^{3}\alpha_i \pi_{t-i}$. Thus, the overall model can be written as
\begin{equation}\label{sect5_eq2}
   \pi_{t+1} = f(\pi_t.\pi_{t-1},\pi_{t-2},\pi_{t-3},\mathbf{x}_t) + v_{t+1},
\end{equation}

where $\mathbf{x}_{t}$ consists of other 131 predictors in the data set along with three of their lags. Therefore, $\mathbf{x}_t$ is a 524-dimensional vector.
We consider one step ahead forecasts computed in a rolling window scheme with 453 observations. Since there are some missing values in the first four years, we start from 1964 to fit model and predict the inflation rate. We use the
observations between May 1964 and Jan 2002 to fit (\ref{sect5_eq2}) and predict the monthly inflation rate starting at time Feb 2002.

We apply the feature screening method Sure Independent Screening (SIS) proposed in \cite{fan2008sure} to reduce the input dimension and select only the top $\gamma\%$ important variables from $\mathbf{x}_t$ for $\gamma = 0, 5, 7.5, 10$. The case of $\gamma = 0$ implies that we do not select any covariates from $\mathbf{x}_t$ and only use four lags of $\pi_t$ as predictors. Then we use the selected covariates and four lags of $\pi_t$ as predictors or the input vector of our neural network. The neural network is selected to have ReLU as the activation function and 100 units in each hidden layers. Two different depths of neural network are selected, one with 3 layers (1 hidden layers) and the other with 6 layers (4 hidden layers). To train the neural network, we use dropout method for hidden layers and gradient descent to update parameters. In each rolling window, 453 observations are divided into two parts: the first 300 observations is treated as training set while the last 153 observations is treated as validation set. The loss function is again (\ref{sect4_eq1}) with $\lambda = 0.1$. We stop training when MSE over the validation set reaches its minimum. The performance is evaluated by MSE of one step ahead forecast. We also show the performance of linear regression (i.e. least squares method) with the same input as neural network. 

% There are 532 predictors in (\ref{sect5_eq2}), we therefore need to reduce the input dimension before training the neural network. We reduce the dimension by SIS (\cite{fan2008sure}) that selects the top $\gamma\%$ important variables from $\mathbf{x}_t$. 

The results are summarized in Table~\ref{tab:my_table1} while Figure \ref{fig4} in the supplementary materials plots the predicted values against the truth. As seen from Table~\ref{tab:my_table1}, neural network has a better performance compared to linear regression method. In other words, the overall prediction error of neural network with three layers is small compared to the ones for linear regression (except for $\gamma = 7.5$). Specifically, in the case of $\gamma = 10$ (i.e. including 52 covariates from vector $\mathbf{x}_t$), the performance of neural network remains satisfactory while linear regression method suffers from large input vector dimension. Finally, note that three layers seems to be enough for this data as increasing the number of layers does not reduce the prediction error.

% In the case $\gamma = 5$, $\lfloor{524\times 5\%\rfloor} = 26$ covariates are selected from $\mathbf{x}_t$. In the case $\gamma = 7.5$, $\lfloor{524\times 10\%\rfloor} = 39$ covariates are selected. In the case $\gamma = 10$, $\lfloor{524\times 10\%\rfloor} = 52$ covariates are selected. Fig \ref{fig3} shows the estimator of inflation rate for three different methods as well as its true value. We observe that 3-layer neural network  clearly outperforms other two methods. For simple autoregressive model, adding more hidden layers doesn't improve the forecasting performance. For linear regression with $\gamma = 10$, since input dimension is relatively high, collinearity happens and leads to a crazy mean square error. While for neural network, such a high dimensional input doesn't hurt its performance. We see that neural network is a powerful and stable method in time series forecasting.

\begin{table}[h]
    \centering
    \begin{tabular}{| c | c | c | c |}
         \hline
  & neural network (3 layers) & neural network (6 layers) & linear regression \\
  \hline
$\gamma\%$ = $0\% $ & $\bold{15.71}$ & 17.44 &  $18.05^{*}$\\
$\gamma\%$ = $5\% $ & $\bold{15.65}$ & 18.37 & 17.05\\
$\gamma\%$ = $7.5\%$& 16.40 & 19.67 & $\bold{15.42}$\\
$\gamma\%$ = $10\%$ & $\bold{16.51}$ & 20.70 & 172.49\\
     \hline
    \end{tabular}
    \caption{Prediction error of inflation rate, i.e. the mean square of the forecasting error. Since the result of neural network depends on initialization, we averaged results across 10 replicates. The entry with asterisk corresponds to the benchmark model (i.e. AR(4) model).} 
    \label{tab:my_table1}
\end{table}

\section{Conclusion}\label{sec:conclusion}

Considering nonparametric regression model with the regression function belonging to a specific family of bounded composite functions, we analyzed the performance of deep feed-forward neural networks with ReLU activation function on estimating such functions. Consistency of prediction error is established under mild conditions on the input data which can include temporal dependence among observations. Interestingly, the consistency rate matches the one for independent data with additional logarithmic factors with respect  to sample size. This result is applicable to a wide range of linear and non-linear time series models including finite lag non-linear auto-regressive models. Then, the result is extended to include a general family of stationary time series models utilizing the Wold decomposition while the consistency rate depends on the decay rate of $\text{AR}(\infty)$ representation coefficients. Relaxing some assumptions in the theoretical analysis including the mixing condition and boundedness of observations in the case of general time series models are interesting future directions. Another limitation of the work is considering only the ReLU as the activation function while extending the results to a more general family of activation functions is a fruitful research direction.

% The theoretical property of neural network in time series setting is barely studied. In this work, we have analyzed the convergence rate of neural network with temporal dependent observations under certain mixing assumption. It is attracting to see that the presence of temporal dependency doesn't have an obvious effect on the convergence rate. In addition, We show the convergence property of neural network when fitting an autoregressive model. The convergence rate depends on the decreasing rate of AR's coefficients. And Properly taking into account more lags of an autoregressive process helps improve the forecasting performance of neural network as sample size increases.

% While the convergence property has been proved, there are some limitations in our analysis. The exponentially $\alpha$-mixing assumption is hard to verify in practice and kind of strong  when compared with other types of mixing. Our theorem is also limited by the boundness of observations in the discussion of autoregressive model. The requirement for boundness is from the restrictions that the neural network estimator $f\in \mathcal{F}(L,p,s,F)$ is also bounded.

% \newpage

\bibliographystyle{unsrt}  
\bibliography{ref}

\begin{thebibliography}{10}

\bibitem{almeida2002predictive}
Jonas~S Almeida.
\newblock Predictive non-linear modeling of complex data by artificial neural
  networks.
\newblock {\em Current opinion in biotechnology}, 13(1):72--76, 2002.

\bibitem{odom1990neural}
Marcus~D Odom and Ramesh Sharda.
\newblock A neural network model for bankruptcy prediction.
\newblock In {\em 1990 IJCNN International Joint Conference on neural
  networks}, pages 163--168. IEEE, 1990.

\bibitem{lisboa2002review}
Paulo~JG Lisboa.
\newblock A review of evidence of health benefit from artificial neural
  networks in medical intervention.
\newblock {\em Neural networks}, 15(1):11--39, 2002.

\bibitem{zhang2004neural}
G~Peter Zhang.
\newblock {\em Neural networks in business forecasting}.
\newblock IGI global, 2004.

\bibitem{brockwell2002introduction}
Peter~J Brockwell and Richard~A Davis.
\newblock {\em Introduction to time series and forecasting}.
\newblock Springer, 2002.

\bibitem{mccaffrey1994convergence}
Daniel~F McCaffrey and A~Ronald Gallant.
\newblock Convergence rates for single hidden layer feedforward networks.
\newblock {\em Neural Networks}, 7(1):147--158, 1994.

\bibitem{barron1993universal}
Andrew~R Barron.
\newblock Universal approximation bounds for superpositions of a sigmoidal
  function.
\newblock {\em IEEE Transactions on Information theory}, 39(3):930--945, 1993.

\bibitem{barron1994approximation}
Andrew~R Barron.
\newblock Approximation and estimation bounds for artificial neural networks.
\newblock {\em Machine learning}, 14(1):115--133, 1994.

\bibitem{bach2017breaking}
Francis Bach.
\newblock Breaking the curse of dimensionality with convex neural networks.
\newblock {\em The Journal of Machine Learning Research}, 18(1):629--681, 2017.

\bibitem{barron2018approximation}
Andrew~R Barron and Jason~M Klusowski.
\newblock Approximation and estimation for high-dimensional deep learning
  networks.
\newblock {\em arXiv preprint arXiv:1809.03090}, 2018.

\bibitem{bauer2019deep}
Benedikt Bauer and Michael Kohler.
\newblock On deep learning as a remedy for the curse of dimensionality in
  nonparametric regression.
\newblock {\em The Annals of Statistics}, 47(4):2261--2285, 2019.

\bibitem{bartlett2017spectrally}
Peter~L Bartlett, Dylan~J Foster, and Matus~J Telgarsky.
\newblock Spectrally-normalized margin bounds for neural networks.
\newblock {\em Advances in neural information processing systems}, 30, 2017.

\bibitem{golowich2018size}
Noah Golowich, Alexander Rakhlin, and Ohad Shamir.
\newblock Size-independent sample complexity of neural networks.
\newblock In {\em Conference On Learning Theory}, pages 297--299. PMLR, 2018.

\bibitem{neyshabur2015norm}
Behnam Neyshabur, Ryota Tomioka, and Nathan Srebro.
\newblock Norm-based capacity control in neural networks.
\newblock In {\em Conference on Learning Theory}, pages 1376--1401. PMLR, 2015.

\bibitem{juditsky2009nonparametric}
Anatoli~B Juditsky, Oleg~V Lepski, and Alexandre~B Tsybakov.
\newblock Nonparametric estimation of composite functions.
\newblock {\em The Annals of Statistics}, 37(3):1360--1404, 2009.

\bibitem{kohler2005adaptive}
Michael Kohler and Adam Krzy{\.z}ak.
\newblock Adaptive regression estimation with multilayer feedforward neural
  networks.
\newblock {\em Nonparametric Statistics}, 17(8):891--913, 2005.

\bibitem{schmidt2020nonparametric}
Johannes Schmidt-Hieber.
\newblock Nonparametric regression using deep neural networks with relu
  activation function.
\newblock {\em The Annals of Statistics}, 48(4):1875--1897, 2020.

\bibitem{truong2021generalization}
Lan~V Truong.
\newblock Generalization error bounds on deep learning with markov datasets.
\newblock {\em arXiv preprint arXiv:2201.11059}, 2021.

\bibitem{fan2019selective}
Jianqing Fan, Cong Ma, and Yiqiao Zhong.
\newblock A selective overview of deep learning.
\newblock {\em arXiv preprint arXiv:1904.05526}, 2019.

\bibitem{choromanska2015loss}
Anna Choromanska, Mikael Henaff, Michael Mathieu, G{\'e}rard~Ben Arous, and
  Yann LeCun.
\newblock The loss surfaces of multilayer networks.
\newblock In {\em Artificial intelligence and statistics}, pages 192--204.
  PMLR, 2015.

\bibitem{doukhan2012mixing}
Paul Doukhan.
\newblock {\em Mixing: properties and examples}, volume~85.
\newblock Springer Science \& Business Media, 2012.

\bibitem{davis2020modeling}
Richard~A Davis and Mikkel~S Nielsen.
\newblock Modeling of time series using random forests: Theoretical
  developments.
\newblock {\em Electronic Journal of Statistics}, 14(2):3644--3671, 2020.

\bibitem{fan2008nonlinear}
Jianqing Fan and Qiwei Yao.
\newblock {\em Nonlinear time series: nonparametric and parametric methods}.
\newblock Springer Science \& Business Media, 2008.

\bibitem{an1996geometrical}
HZ~An and FC~Huang.
\newblock The geometrical ergodicity of nonlinear autoregressive models.
\newblock {\em Statistica Sinica}, pages 943--956, 1996.

\bibitem{wang2021consistent}
Jiang Wang and Dimitris~N Politis.
\newblock Consistent autoregressive spectral estimates: Nonlinear time series
  and large autocovariance matrices.
\newblock {\em Journal of Time Series Analysis}, 42(5-6):580--596, 2021.

\bibitem{kreiss2011range}
Jens-Peter Kreiss, Efstathios Paparoditis, and Dimitris~N Politis.
\newblock On the range of validity of the autoregressive sieve bootstrap.
\newblock {\em The Annals of Statistics}, 39(4):2103--2130, 2011.

\bibitem{medeiros2016l1}
Marcelo~C Medeiros and Eduardo~F Mendes.
\newblock $\ell_1$-regularization of high-dimensional time-series models with
  non-gaussian and heteroskedastic errors.
\newblock {\em Journal of Econometrics}, 191(1):255--271, 2016.

\bibitem{wu2016performance}
Wei-Biao Wu and Ying~Nian Wu.
\newblock Performance bounds for parameter estimates of high-dimensional linear
  models with correlated errors.
\newblock {\em Electronic Journal of Statistics}, 10(1):352--379, 2016.

\bibitem{fan2008sure}
Jianqing Fan and Jinchi Lv.
\newblock Sure independence screening for ultrahigh dimensional feature space.
\newblock {\em Journal of the Royal Statistical Society: Series B (Statistical
  Methodology)}, 70(5):849--911, 2008.

\bibitem{merlevede2009bernstein}
Florence Merlev{\`e}de, Magda Peligrad, and Emmanuel Rio.
\newblock Bernstein inequality and moderate deviations under strong mixing
  conditions.
\newblock In {\em High dimensional probability V: the Luminy volume}, pages
  273--292. Institute of Mathematical Statistics, 2009.

\bibitem{de36general}
VH~De~La~Pena.
\newblock A general class of exponential inequalities for martingales and
  ratios. 1999.
\newblock {\em Ann. Probab}, 36:1902--1938.

\end{thebibliography}

\newpage

\renewcommand{\thesection}{\Alph{section}}

\setcounter{section}{0}
\section*{Appendix}
% the \\ insures the section title is centered below the phrase: AppendixA

We provide proofs of main theorems in Section~\ref{supp:sec:sim:1.2} while some useful Lemmas are stated and proved in Section~\ref{supp:sec:sim:1.1}. In Section \ref{supp:sec:sim:2}, additional details related to the numerical experiments are provided. We also provide a numerical comparison between feed-forward deep neural networks (DNN) and Long Short-Term Memory networks (LSTM) in Section \ref{supp:sec:sim:3}.

% . Some useful Lemmas is proved in Section \ref{supp:sec:sim:1.1}. These lemmas will be used in Section \ref{supp:sec:sim:1.2} and \ref{supp:sec:sim:1.3}. 

% \textcolor{red}{where do we use all those 5 assumptions we state in the main? They need to be mentioned in the statement of lemmas ... for example, for lemma 4, I added that we need assumptions~1-3, please do the same for all other lemmas ...}

\section{Proofs in Section 3}\label{supp:sec:sim:1}
\subsection{Useful Lemmas and their proofs}\label{supp:sec:sim:1.1}

\begin{lemma}\label{lemma1}
Assume that $\{\mathbf{X}_t\}_{t\in \mathbb{Z}}$ satisfies Assumption \ref{assum2}, that is, $\{\mathbf{X}_t\}_{t\in \mathbb{Z}}$ is a stationary and exponential $\alpha-$mixing process. Further, assume $f$ is a measurable function satisfying $\Vert f(\cdot)\Vert_{\infty} < M$. Then 
\[
|\mathrm{Cov}(f(\mathbf{X}_i),f(\mathbf{X}_j))| \leq 11M^2\mathrm{exp}(-c|i-j|/3),
\]
for some positive constant $c$ depending on $\tilde{c}$.
\end{lemma}

\begin{proof}
Consider the set of grid points 
\begin{equation*}
 \begin{split}
&D(m):=\{(a_k,b_l),\ k,l=1,2,\cdots,2m+1 \}\\
&a_k= M(k-m-1)/m ,\ b_l= M(l-m-1)/m, 
\end{split}
\end{equation*}
where $m$ is some positive integer. Then we have
\begin{equation}\label{supple_eq1}
    \begin{split}
&\bigg|\mathbb{E}(f(\mathbf{X}_i)f(\mathbf{X}_j)) - \sum_{k=1}^{2m}\sum_{l=1}^{2m}(\frac{a_k+a_{k+1}}{2})(\frac{b_l+b_{l+1}}{2})\mathbb{P}(a_k\leq f(\mathbf{X}_i)<a_{k+1}, b_l\leq f(\mathbf{X}_j) <b_{l+1})\bigg|\\
&\leq
\sum_{k=1}^{2m}\sum_{l=1}^{2m}\mathbb{P}(a_k\leq f(\mathbf{X}_i)<a_{k+1}, b_l\leq f(\mathbf{X}_j) <b_{l+1})\frac{M^2}{m}\\
&= \frac{M^2}{m}.
    \end{split}
\end{equation}
Similar to (\ref{supple_eq1}) we can prove that
\begin{equation*}
    \begin{split}
\bigg| \mathbb{E}[f(\mathbf{X}_i)] - \sum_{k=1}^{2m}\frac{a_k+a_{k+1}}{2}\mathbb{P}(a_k\leq f(\mathbf{X}_i)<a_{k+1}) \bigg| \leq \frac{M}{2m}.
    \end{split}
\end{equation*}

To simplify the notation, we let $A_i:=\sum_{k=1}^{2m}\frac{a_k+a_{k+1}}{2}\mathbb{P}(a_k\leq f(\mathbf{X}_i)<a_{k+1})$. Then, we have

\begin{equation}\label{supple_eq2}
\begin{split}
\bigg| \mathbb{E}[f(\mathbf{X}_i)]\mathbb{E}[f(\mathbf{X}_j)] - A_iA_j \bigg|
&\leq
\bigg| \mathbb{E}[f(\mathbf{X}_i)](\mathbb{E}[f(\mathbf{X}_j)] - A_j) \bigg| 
+
\bigg| (\mathbb{E}[f(\mathbf{X}_j)] - A_j)A_i \bigg|\\
&\leq
\frac{M^2}{2m}+\frac{M}{2m}(M+\frac{M}{2m}) = \frac{M^2}{m}+\frac{M^2}{4m^2}.
\end{split}
\end{equation}

Since  $\{\mathbf{X}_t\}_{t\in \mathbb{Z}}$ is an exponential $\alpha-$mixing sequence, we know that there exists some positive constant $c$ such that
\begin{equation*}
    \begin{split}
&\mathbb{P}(a_k\leq f(\mathbf{X}_i)<a_{k+1}, b_l\leq f(\mathbf{X}_j) <b_{l+1})\\
&\leq
\mathrm{exp}(-c|i-j|)+\mathbb{P}(a_k\leq f(\mathbf{X}_i)<a_{k+1})\mathbb{P}(b_l\leq f(\mathbf{X}_j) <b_{l+1}).
    \end{split}
\end{equation*}

Therefore

\begin{equation}\label{supple_eq3}
    \begin{split}
&\bigg|\sum_{k=1}^{2m}\sum_{l=1}^{2m}(\frac{a_k+a_{k+1}}{2})(\frac{b_l+b_{l+1}}{2})\mathbb{P}(a_k\leq f(\mathbf{X}_i)<a_{k+1}, b_l\leq f(\mathbf{X}_j) <b_{l+1}) - A_iA_j\bigg|\\
 &\leq
\sum_{k=1}^{2m}\sum_{l=1}^{2m}|(\frac{a_k+a_{k+1}}{2})(\frac{b_l+b_{l+1}}{2})|\mathrm{exp}(-c|i-j|) \\
&\leq 4m^2M^2\mathrm{exp}(-c|i-j|).
    \end{split}
\end{equation}

From (\ref{supple_eq1}),(\ref{supple_eq2}) and (\ref{supple_eq3}), we have that
\begin{equation*}
\begin{split}
\bigg| \mathbb{E}[f(\mathbf{X}_i)f(\mathbf{X}_j)] - \mathbb{E}[f(\mathbf{X}_i)]\mathbb{E}[f(\mathbf{X}_j)] \bigg| \leq \frac{2M^2}{m} + \frac{M^2}{4m^2} + 4m^2M^2\mathrm{exp}(-c|i-j|).
\end{split}
\end{equation*}

With the choice of $m=\lfloor \mathrm{exp}(c|i-j|/3) \rfloor$,

\begin{equation*}
\begin{split}
\bigg| \mathbb{E}[f(\mathbf{X}_i)f(\mathbf{X}_j)] - \mathbb{E}[f(\mathbf{X}_i)]\mathbb{E}[f(\mathbf{X}_j)] \bigg| &\leq 10M^2\mathrm{exp}(-c|i-j|/3)  + \frac{M^2}{4}\mathrm{exp}(-2c|i-j|/3) \\
&\leq 11M^2\mathrm{exp}(-c|i-j|/3).
\end{split}
\end{equation*}
\end{proof}

\begin{lemma}\label{lemma2}
Let $\{\mathbf{X}_t\}_{t\in \mathbb{Z}}$ and $f$ be as in lemma \ref{lemma1} while $\{a_n\}_{n\in\mathbb{Z}}$ is a sequence of real numbers such that $a_n\leq n^{\alpha}$ for some positive $\alpha$. Let $Y_{ni} = a_nf(\mathbf{X}_i)$, then
\[
\mathrm{Var}(Y_{n0}) + 2\sum_{i>0}|\mathrm{Cov}(Y_{n0},Y_{ni})|
\leq
([24\alpha\mathrm{log}(n)/c]+3)\mathrm{Var}(Y_{n0}) + 22M^2\frac{1}{n^{2\alpha}(\mathrm{exp}(\frac{c}{3})-1)},
\]
where c is the same constant as in lemma \ref{lemma1} depending on $\tilde{c}$ (the $\alpha$-mixing exponent).
\end{lemma}

\begin{proof}
Let  $k=[\frac{12\alpha}{c}\mathrm{log}(n)]+2$. Notice that $Y_{ni} < n^{\alpha}M$ for all $i\in \mathbb{Z}$. From lemma \ref{lemma1}, we have that
\begin{equation*}
\begin{split}
2\sum_{i\geq k}|\mathrm{Cov}(Y_{n0},Y_{nk})|
\leq
22\sum_{i\geq k}(n^{\alpha}M)^2\mathrm{exp}(-ci/3)
=
22n^{2\alpha}M^2\frac{\mathrm{exp}(-\frac{c}{3}(k-1))}{\mathrm{exp}\frac{c}{3}-1}\\
\leq
22M^2\frac{1}{n^{2\alpha}\mathrm{exp}(\frac{c}{3}-1)}.
\end{split}
\end{equation*}

Therefore

\begin{equation*}
\begin{split}
\mathrm{Var}(Y_{n0})+2\sum_{i> 0}|\mathrm{Cov}(Y_{n0},Y_{nk})|\leq
(2k-1)\mathrm{Var}(Y_{n0}) + 22M^2\frac{1}{n^{2\alpha}(\mathrm{exp}(\frac{c}{3})-1)}.
\end{split}
\end{equation*}
\end{proof}

Let $\mathcal{F}$ be a class of function. We define $\mathcal{N} (\delta, \mathcal{F}, \Vert \cdot \Vert_{\infty})$ to be the covering number, that is, the minimal number of $\Vert \cdot \Vert_{\infty}$-balls with radius $\delta$ that covers $\mathcal{F}$.

\begin{lemma}\label{lemma3}
Consider the d-variate nonparametric regression model with unknown regression function $f_0$, $Y_i=f_0(X_i)+\epsilon_i$, satisfying Assumptions~1-3. Let $\widehat{f}$ be any estimator taking values in $\mathcal{F}$. Define
\[
\Delta_n := \Delta_n(\widehat{f},f_0,\mathcal{F}) := \mathbb{E}_{f_0}\left[\frac{1}{n}\sum_{i=1}^{n}(Y_i-\widehat{f}(\mathbf{X}_i))^2 -\mathop{\mathrm{inf}}_{f\in \mathcal{F}} \frac{1}{n}\sum_{i=1}^{n}(Y_i-f(\mathbf{X}_i))^2 \right]
\]
and assume $\{f_0\}\cup \mathcal{F}\subset \{f: [0,1]^d\rightarrow [-F,F]\}$ for some $F\geq 1$. If $N_n:= \mathcal{N}(\delta,\mathcal{F},\Vert\cdot\Vert_{\infty})\geq 3$, then,
\begin{equation*}
\begin{split}
&(1-\epsilon)^2 \Delta_n - C(F,\sigma^2,c)\frac{\mathrm{log}^4n\mathrm{log}N_n}{n\epsilon}-\delta C(F,\sigma^2,c)\mathrm{log}^2n\\
&\leq R(\widehat{f},f_0)\leq\\
&(1+\epsilon)^2\left(\mathop{\mathrm{inf}}_{f^{*}\in \mathcal{F}} \Vert f^{*}-f_0\Vert_{\infty}^{2}+ \Delta_n(\widehat{f},f)+C(F,\sigma^2,c)\delta \mathrm{log}^2n\right)\\
&+
\frac{(1+\epsilon)^3}{\epsilon}C(F,\sigma^2,c)\frac{\mathrm{log}^4n\mathrm{log}N_n}{n},
\end{split}
\end{equation*}

where $C(F,\sigma^2,c)$ is defined as a constant depending only on $F,\sigma^2$, and $c$ (the same constant in lemma \ref{lemma1}).
\end{lemma}

\begin{proof}[Proof.]
Throughout the proof we write $\mathbb{E}= \mathbb{E}_{f_0}$. Define $\Vert g\Vert_n^2:= \frac{1}{n}\sum_{i=1}^{n}g(\mathbf{X}_i)^2 $. For any estimator $\tilde{f}$, we introduce $\widehat{R}_n(\tilde{f},f_0):= \mathbb{E}[\Vert \tilde{f}-f_0 \Vert_n^2]$ for the empirical risk. In the first step, we show that we can
restrict ourselves to the case $\mathrm{log}N_n \leq n$. Since $R(\widehat{f},f_0)\leq 4F^2 $, the upper bound trivially holds if $\mathrm{log}N_n\geq n$. To see that also the lower bound is trivial in this case, let $\tilde{f}\in \mathrm{argmin}_{f\in\mathcal{F}}\sum_{i=1}^{n}(Y_i-f(\mathbf{X}_i))^2$ be a (global) empirical
risk minimizer. Observe that
\begin{equation*}
\begin{split}
\widehat{R}_n(\widehat{f},f_0) - \widehat{R}_n(\tilde{f},f_0) = \Delta_n + \mathbb{E}\left[\frac{2}{n}\sum_{i=1}^{n}\epsilon_i\widehat{f}(\mathbf{X}_i)\right] - \mathbb{E}\left[ \frac{2}{n}\sum_{i=1}^{n}\epsilon_i\tilde{f}(\mathbf{X}_i)\right].
\end{split}
\end{equation*}

From this equation, it follows that $\Delta_n \leq 8F^2$ and this implies the lower
bound in the statement of the lemma for $\mathrm{log} N_n \geq n$. We may therefore
assume $\mathrm{log} N_n \leq n$. The proof is divided into four parts which are denoted
by (\uppercase\expandafter{\romannumeral1}$\sim$\uppercase\expandafter{\romannumeral4}) \\
(\uppercase\expandafter{\romannumeral1}): We relate the risk $R( \widehat{f}, f_0)= \mathbb{E}[(\widehat{f}(\mathbf{X}) - f_0(
\mathbf{X}))^2]$ to its empirical
counterpart  $\widehat{R}_n( \widehat{f}, f_0)$ via the inequalities

\begin{equation*}
\begin{split}
&(1-\epsilon)\widehat{R}_n(\widehat{f},f_0) -C_F\frac{\mathrm{log}^2n\mathrm{log}N_n}{n} -C_F \delta\mathrm{log}^2n
-\frac{1}{\epsilon} \frac{C_FF^2\mathrm{log}^4n\mathrm{log}N_n}{n}\\
&\leq R(\widehat{f},f_0)\leq\\
&(1+\epsilon)\widehat{R}_n(\widehat{f},f_0)+(1+\epsilon)\left( C_F\frac{\mathrm{log}^2n\mathrm{log}N_n}{n} + C_F\delta \mathrm{log}^2n  \right)\\
&+\frac{(1+\epsilon)^2}{\epsilon} \frac{C_FF^2\mathrm{log}^4n\mathrm{log}N_n}{n},
\end{split}
\end{equation*}
where $C_F$ is some constant which depends on $F $ and $c$.\\

(\uppercase\expandafter{\romannumeral2}): For any estimator $\tilde{f}$ taking values in $\mathcal{F}$,
\begin{equation*}
\begin{split}
\mathbb{E}\left[\bigg|\frac{2}{n}\sum_{i=1}^{n}\epsilon_i\tilde{f}(\mathbf{X}_i)\bigg|\right]
&\leq 
\delta C(F,\sigma^2,c)\mathrm{log}^2n + C(F,\sigma^2,c)\mathrm{log}^4n\frac{\mathrm{log}N_n}{n}\\
&\quad\quad\quad
+C(F,\sigma^2,c)\sqrt{\frac{\mathrm{log}^4n\mathrm{log}N_n}{n}}\widehat{R}_n^{1/2}(\tilde{f},f_0).
\end{split}
\end{equation*}

(\uppercase\expandafter{\romannumeral3}): We have
\begin{equation*}
\begin{split}
\widehat{R}_n(\widehat{f},f_0)
&\leq
(1+\epsilon)\bigg[
\mathrm{inf}_{f\in\mathcal{F}}\mathbb{E}[(f(\mathbf{X})-f_0(\mathbf{X}))^2]
+\Delta_n+
\delta C(F,\sigma^2,c)\mathrm{log}^2n \\
&\quad\quad
+ 
C(F,\sigma^2,c)\mathrm{log}^4n\frac{\mathrm{log}N_n}{n}\bigg]
+
\frac{(1+\epsilon)^2}{\epsilon}C^2(F,\sigma^2,c)\mathrm{log}^4n\frac{\mathrm{log}N_n}{n}. 
\end{split}
\end{equation*}

(\uppercase\expandafter{\romannumeral4}): We have
\begin{equation*}
\begin{split}
\widehat{R}_n(\widehat{f},f_0)\geq
(1-\epsilon)(\Delta_n- C(F,\sigma^2,c)\frac{\mathrm{log}^4n\mathrm{log}N_n}{n\epsilon}-2\delta C(F,\sigma^2,c)\mathrm{log}^2n).
\end{split}
\end{equation*}

Combining (\uppercase\expandafter{\romannumeral1}) and (\uppercase\expandafter{\romannumeral4}) gives the lower bound of the assertion. The upper bound follows from (\uppercase\expandafter{\romannumeral1}) and (\uppercase\expandafter{\romannumeral3}).\\

(\uppercase\expandafter{\romannumeral1}): Given a minimal $\delta$-covering of $\mathcal{F}$, denote the centers of the balls by
$f_j$ . By construction there exists a (random) $j^{*}$ such that $\Vert \widehat{f}-f_{j^{*}}\Vert_{\infty}\leq \delta$. Without loss of generality, we can assume that $\Vert f_j\Vert_{\infty}\leq F $. Generate i.i.d. random variables $\{\mathbf{X}_i^{'}$, $i=1,\cdots,n\}$ with the same distribution as $\mathbf{X}$ ($\mathbf{X}\overset{D}{=}\mathbf{X}_i$) and independent of $\{\mathbf{X}_i,i=1,\cdots,n\}$. Using that $\Vert f_j\Vert_{\infty}, \Vert f_0\Vert_{\infty}, \delta \leq F$,
\begin{equation*}
\begin{split}
&|R(\widehat{f},f_0)-\widehat{R}_n(\widehat{f},f_0)|\\
&= \bigg|\mathbb{E}\left[\frac{1}{n}\sum_{i=1}^{n}(\widehat{f}(\mathbf{X}_i^{'})-f_0(\mathbf{X}_i^{'}))^2-\frac{1}{n}\sum_{i=1}^{n}(\widehat{f}(\mathbf{X}_i)-f_0(\mathbf{X}_i))^2\right]\bigg|\\
&\leq 
\mathbb{E}\left[\bigg|\frac{1}{n}\sum_{i=1}^{n}g_{j^{*}}(\mathbf{X}_i,\mathbf{X}_i^{'}) \bigg|\right] + 9\delta F,
\end{split}
\end{equation*}

% $\gamma_j$ \textcolor{red}{do you mean $\gamma_j$?} for $j = j^{*}$, which is the same \textcolor{red}{they are not the same, one is conditional expectation and the other one is not}\textcolor{blue}{(the conditional expectation due to the fact that $j^{*}$ depends on observations)} as

with $g_{j^{*}}(\mathbf{X}_i,\mathbf{X}_i^{'}):=(f_{j^{*}}(\mathbf{X}_i^{'})-f_0(\mathbf{X}_i^{'}))^2-(f_{j^{*}}(\mathbf{X}_i)-f_0(\mathbf{X}_i))^2$. Define $g_j$ in the same way with $f_j^{*}$ replaced by $f_j$. Similarly, set $\gamma_j :=\sqrt{n^{-1}\mathrm{log}N_n} \vee \mathbb{E}^{1/2}[(f_j(\mathbf{X})-f_0(\mathbf{X}))^2]$ and define $\gamma^{*}$ as 
\begin{equation}\label{supp_eq0}
\begin{split}
\gamma^{*} 
&= \sqrt{n^{-1}\mathrm{log}N_n}\vee \mathbb{E}^{1/2}[(f_{j^{*}}(\mathbf{X})-f_0(\mathbf{X}))^2|\{(\mathbf{X}_i,Y_i)\}_{i=1}^{\infty}]\\
&\leq \sqrt{n^{-1}\mathrm{log}N_n} + \mathbb{E}^{1/2}[(\widehat{f}(\mathbf{X})-f_0(\mathbf{X}))^2|\{(\mathbf{X}_i,Y_i)\}_{i=1}^{\infty}]+\delta,
\end{split}
\end{equation}

where the last part follows from triangle inequality and $f_{j^{*}}-\widehat{f} \leq \delta $.\\
For random variables $U_1$, $T_1$, Cauchy-Schwarz inequality gives $\mathbb{E}[U_1T_1]\leq \mathbb{E}^{1/2}[U_1^2]\mathbb{E}^{1/2}[T_1^2] $. Choose $U_1=\mathbb{E}^{1/2}[(\widehat{f}(\mathbf{X})-f_0(\mathbf{X}))^2|\{(\mathbf{X}_i,Y_i)\}_{i=1}^{\infty}]$ and $T_1 = \mathop{\mathrm{max}}_{j}|\sum_{i=1}^{n}g_j(\mathbf{X}_{i},\mathbf{X}_{i^{'}})/\gamma_j F| $. Using that $\mathbb{E}[U_1^2] = R( \widehat{f}, f_0)$,

\begin{equation}\label{eq1}
    \begin{split}
&| R(\widehat{f},f_0) - \widehat{R}_n(\widehat{f},f_0) |\\
&\leq
\frac{F}{n}R(\widehat{f},f_0)^{\frac{1}{2}}\mathbb{E}^{\frac{1}{2}}[T_1^2] + \frac{F}{n}(\sqrt{\frac{\mathrm{log}N_n}{n}}+\delta)\mathbb{E}[T_1] + 9\delta F.
    \end{split}
\end{equation}

Next, we need to estimate the upper bound for $\mathbb{E}[T]$ and $\mathrm{E}[T^2]$. To simplify the notation, we let $v_j^2=\mathrm{Var}(g_j(\mathbf{X}_i,\mathbf{X}_i^{'})/\gamma_j F)$ and $\tilde{v}_j^2 = \mathrm{Var}\left(\frac{g_j(\mathbf{X}_i,\mathbf{X}_i^{'})}{\gamma_jF}\right) + 2\sum_{k>i}\mathrm{Cov}\left( \frac{g_j(\mathbf{X}_k,\mathbf{X}_k^{'})}{\gamma_jF}, \frac{g_j(\mathbf{X}_i,\mathbf{X}_i^{'})}{\gamma_jF}\right)$. We know that $ |g_j(\mathbf{X}_i,\mathbf{X}_i^{'})/F|\leq 4F$ and $1/\gamma_j\leq n^{1/2}$. From lemma \ref{lemma2}, we can derive

\begin{equation*}
    \begin{split}
\tilde{v}_j^2\leq
([12\mathrm{log}(n)/c]+3)v_j^2+22(4F)^2\frac{1}{n(\mathrm{exp}(c/3)-1)}.
    \end{split}
\end{equation*}

Since $v_j^2 = 2\mathrm{Var}((f_j(\mathbf{X}_i)-f_0(\mathbf{X}_i))^2/\gamma_jF)
\leq 2\mathbb{E}[(f_j(\mathbf{X}_i)-f_0(
\mathbf{X}_i))^4]/(\gamma_j^2F^2)
\leq 8$, 
we conclude that $\tilde{v}_j^2\leq C_1\mathrm{log}(n) + C_2F^2/n$, where $C_1$ and $C_2$ are constants which only depend on $c$. Observe that $\mathbb{E}[g_j(\mathbf{X}_i,\mathbf{X}_i^{'})]=0$, $|\frac{g_j(X_i,X_i^{'})}{\gamma_jF}|\leq 4F/\gamma_j$. Also, we know that $\{g_j(\mathbf{X}_i,\mathbf{X}_{i^{'}})/\gamma_jF,\ i=1,\cdots,n\}$ is an exponentially $\alpha$-mixing process. By Bernstein inequality (Theorem 2, \cite{merlevede2009bernstein}) with a union bound over $j$, we have

\begin{equation*}
    \begin{split}
&\mathbb{P}(T_1\geq t) \leq 1 \wedge \left(2N_n \mathop{\mathrm{max}}_j\mathrm{exp}\left(-\frac{\tilde{c}t^2}{n\tilde{v}_j^2 +\frac{16F^2}{\gamma_j^2}+t\frac{4F}{\gamma_j}(\mathrm{log}n)^2} \right)\right)\\
&\leq
1 \wedge \left(2N_n \mathop{\mathrm{max}}_j\mathrm{exp}\left(-\frac{\tilde{c}t^2}{C_1n\mathrm{log}(n) +C_2F^2 +16F^2\frac{n}{\mathrm{log}(N_n)}+4tF(\mathrm{log}n)^2\sqrt{\frac{n}{\mathrm{log}(N_n)}}} \right)\right)\\
&\leq
1 \wedge \left(2N_n \mathop{\mathrm{max}}_j\mathrm{exp}\left(-\frac{\tilde{c}t^2}{n(C_1\mathrm{log}(n) + (16+C_2)F^2)+4tF(\mathrm{log}n)^2\sqrt{\frac{n}{\mathrm{log}(N_n)}}} \right)\right).
\end{split}
\end{equation*}

Therefore we can estimate the upper of $\mathbb{E}[T_1]$ by $\mathbb{P}(T_1>t)$
\begin{equation*}
\begin{split}
\mathbb{E}[T_1] 
&= \int_0^{\infty}\mathbb{P}(T_1>t)dt\\
&\leq
\theta\sqrt{n\mathrm{log}(N_n)} +2N_n\int_{\theta\sqrt{n\mathrm{log}(N_n)}}^{\infty}\mathrm{exp}\left(-\frac{\tilde{c}t}{\frac{((16+C_2)F^2+C_1\mathrm{log}(n))n}{\theta\sqrt{n\mathrm{log}(N_n)}}+4\sqrt{\frac{n}{\mathrm{log}(N_n)}}(\mathrm{log}n)^2F}\right)\\
&=
\theta\sqrt{n\mathrm{log}(N_n)} + 2N_n\frac{1}{\tilde{c}}\left(\frac{((16+C_2)F^2+C_1\mathrm{log}(n))n}{\theta\sqrt{n\mathrm{log}(N_n)}}+4\sqrt{\frac{n}{\mathrm{log}(N_n)}}(\mathrm{log}n)^2F\right)\\
&\quad\quad\quad\quad\quad\quad\quad\quad
\mathrm{exp}\left(-\frac{\tilde{c}\theta\sqrt{n\mathrm{log}(N_n)}}{\frac{((16+C_2)F^2+C_1\mathrm{log}(n))n}{\theta\sqrt{n\mathrm{log}(N_n)}}+4\sqrt{\frac{n}{\mathrm{log}(N_n)}}(\mathrm{log}n)^2F}\right).
    \end{split}
\end{equation*}

Let $\theta = \sqrt{\frac{(32+2C_2)F^2+2C_1\mathrm{log}(n)}{\tilde{c}}} \vee \frac{8F(\mathrm{log}n)^2}{\tilde{c}}$, we have
\begin{equation*}
    \begin{split}
\mathbb{E}[T] 
&\leq
\theta\sqrt{n\mathrm{log}(N_n)} + \left(\frac{(32+2C_2)F^2+2C_1\mathrm{log}(n)}{\theta\tilde{c}} +\frac{8F(\mathrm{log}n)^2}{\tilde{c}} \right)\sqrt{\frac{n}{\mathrm{log}(N_n)}}\\
&=
\theta \sqrt{n\mathrm{log}(N_n)} +A_1\sqrt{\frac{n}{\mathrm{log}(N_n)}}.
    \end{split}
\end{equation*}

Next we estimate the upper bound of $\mathbb{E}[T_1^2] $ in a similar way.
\begin{equation*}
\begin{split}
&\mathbb{E}[T_1^2] = \int_0^{\infty}\mathbb{P}(T_1>\sqrt{t})dt\\
&\leq \theta^2 n \mathrm{log}(N_n) + 2N_n\int_{\theta^2 n\mathrm{log}(N_n)}^{\infty}\mathrm{exp}\left(-\frac{\tilde{c}\sqrt{t}}{\frac{((16+C_2)F^2+C_1\mathrm{log}(n))n}{\theta\sqrt{n\mathrm{log}(N_n)}}+4\sqrt{\frac{n}{\mathrm{log}(N_n)}}(\mathrm{log}n)^2F}\right) dt.
\end{split}
\end{equation*}

Using the fact that $\int_{b^2}^{\infty} \mathrm{exp}(-\sqrt{u}a)du = (2ab+1)\mathrm{exp}(-ab)/a^2$ and with the same choice of $\theta$ as above, we estimate the upper bound of $\mathbb{E}[T_1]$. Also, we can prove that
\begin{equation*}
\begin{split}
\mathbb{E}[T_1^2]
&\leq 
\theta^2n\mathrm{log}(N_n)+6n\left(\frac{(16+C_2)F^2+C_1\mathrm{log}(n)}{\theta\tilde{c}^2}+\frac{4\mathrm{log}^2(n)F}{\tilde{c}^2}\right)\\
&=
\theta^2n\mathrm{log}(N_n) + 3nA_1/\tilde{c}\\
&\leq
(\theta^2 + \frac{3A_1}{\tilde{c}})n\mathrm{log}N_n.
\end{split}
\end{equation*}

With eq(\ref{eq1}) and the upper bound for $\mathbb{E}T_1$ and $\mathbb{E}T_1^2$, we have
\begin{equation}\label{eq2}
\begin{split}
|R(\widehat{f},f_0)-\widehat{R}_n(\widehat{f},f_0)|
&\leq
\frac{F}{n}R(\widehat{f},f_0)^{\frac{1}{2}}\sqrt{\theta^2n\mathrm{log}(N_n)+3nA_1/\tilde{c}} +\\
& \quad \quad \frac{F}{n}(\sqrt{\frac{\mathrm{log}N_n}{n}}+\delta)\left(\theta \sqrt{n\mathrm{log}N_n} +A_1\sqrt{\frac{n}{\mathrm{log}N_n}}\right) + 9\delta F\\
&=
\frac{F}{n}R(\widehat{f},f_0)^{\frac{1}{2}}\sqrt{\theta^2n\mathrm{log}(N_n)+3nA_1/\tilde{c}} +\frac{\theta F \mathrm{log}N_n}{n}+\frac{A_1F}{n}+\\
&\quad \quad F\delta(\theta\sqrt{\frac{\mathrm{log}N_n}{n}}+A_1\sqrt{\frac{1}{n\mathrm{log}N_n}}+9)\\
&\leq
\frac{F}{n}R(\widehat{f},f_0)^{\frac{1}{2}}\sqrt{(\theta^2 + \frac{3A_1}{\tilde{c}})n\mathrm{log}N_n}+\frac{(\theta+A_1)F\mathrm{log}N_n}{n}+\\
&\quad \quad(\theta + A_1+9)F\delta.
\end{split}
\end{equation}

Since there exists some constant such that $\frac{(\theta+A_1+9)F}{\mathrm{log}^2n} \vee \frac{\theta^2+3A_1/\tilde{c}}{\mathrm{log}^4n}\leq C_F$, where $C_F$ depends on $F$ and $c$, eq(\ref{eq2}) can be simplified as 
\begin{equation*}
\begin{split}
|R(\widehat{f},f_0)-\widehat{R}_n(\widehat{f},f_0)|
\leq
&\frac{F}{n}R(\widehat{f},f_0)^{\frac{1}{2}}\sqrt{C_F n\mathrm{log}^4n\mathrm{log}N_n}+C_F\frac{\mathrm{log}^2n\mathrm{log}N_n}{n}+\\
&\quad \quad C_F \delta\mathrm{log}^2n.
\end{split}
\end{equation*}

From eq(43) in \cite{schmidt2020nonparametric}, we know that for positive real numbers $a, b, c, d$ being such that $|a-b|\leq 2\sqrt{a}c+d$, then for any $0<\epsilon<1$,

\begin{equation}\label{eq3}
\begin{split}
(1-\epsilon)b-d-\frac{c^2}{\epsilon}
\leq
a
\leq
(1+\epsilon)(b+d)+\frac{(1+\epsilon)^2}{\epsilon}c^2.
\end{split}
\end{equation}

Using (\ref{eq3}) with $a=R(\widehat{f},f_0)$ and $b=\widehat{R}(\widehat{f},f_0)$, we can derive the following bounds for $R(\widehat{f},f_0)$ from (\ref{eq2}):

\begin{equation*}
\begin{split}
&(1-\epsilon)\widehat{R}_n(\widehat{f},f_0) -C_F\frac{\mathrm{log}^2n\mathrm{log}N_n}{n} -C_F \delta\mathrm{log}^2n
-\frac{1}{\epsilon} \frac{C_FF^2\mathrm{log}^4n\mathrm{log}N_n}{n}\\
&\leq R(\widehat{f},f_0)\leq\\
&(1+\epsilon)\widehat{R}_n(\widehat{f},f_0)+(1+\epsilon)\left( C_F\frac{\mathrm{log}^2n\mathrm{log}N_n}{n} + C_F\delta \mathrm{log}^2n  \right)\\
&+\frac{(1+\epsilon)^2}{\epsilon} \frac{C_FF^2\mathrm{log}^4n\mathrm{log}N_n}{n}.
\end{split}
\end{equation*}

(\uppercase\expandafter{\romannumeral2}): Similar to  the proof of (\uppercase\expandafter{\romannumeral1}), there exists a random $j^{*}$ such that $\Vert f_{j^{*}}-\tilde{f} \Vert_{\infty}\leq \delta$. We have $|\mathbb{E}[\sum_{i=1}^n\epsilon_i(\tilde{f}(\mathbf{X}_i)-f_{j^{*}}(\mathbf{X}_i))]| \leq \delta \mathbb{E}[\sum_{i=1}^n|\epsilon_i|]\leq n\delta$. Since $\mathbb{E}[\epsilon_if_0(\mathbf{X}_i)]= \mathbb{E}[\mathbb{E}[\epsilon_if_0(X_i)|X_i]] = 0$, we also find

\begin{equation}\label{eq4}
\begin{split}
\bigg| \mathbb{E}[\frac{2}{n}\sum_{i=1}^n \epsilon_i\tilde{f}(\mathbf{X}_i)]\bigg|
&=\bigg| \mathbb{E}[\frac{2}{n}\sum_{i=1}^n\epsilon_i(\tilde{f}(\mathbf{X}_i)-f_0(\mathbf{X}_i))] \bigg| \\
&\leq
2\delta + \frac{2}{n}\mathbb{E}\bigg| \sum_{i=1}^n \epsilon_i(f_{j^{*}}(\mathbf{X}_i)-f_0(\mathbf{X}_i)) \bigg|.
\end{split}
\end{equation}

Recall that $\gamma_j:=  \sqrt{n^{-1}\mathrm{log}N_n}\vee\mathbb{E}^{1/2}[(f_j(\mathbf{X})-f(\mathbf{X}))^2]$ and the definition of $\gamma^{*}$ is 

\begin{equation*}
\begin{split}
\gamma^{*} 
&= \sqrt{n^{-1}\mathrm{log}N_n}\vee \mathbb{E}^{1/2}[{(f_{j^{*}}(\mathbf{X})-f_0(\mathbf{X}))^2|\{(\mathbf{X}_i,Y_i)\}_{i=1}^{\infty}}]\\
&\leq
\sqrt{n^{-1}\mathrm{log}N_n}+\mathbb{E}^{1/2}[(\tilde{f}(\mathbf{X})-f_0(\mathbf{X}))^2|\{(\mathbf{X}_i,Y_i)\}_{i=1}^{\infty}]+\delta.
\end{split}
\end{equation*}

Using Cauchy-Schwarz inequality $\mathbb{E}[U_2T_2]\leq \mathbb{E}^{1/2}[U_2^2]\mathbb{E}^{1/2}[T_2^2]$ with $T_2=\mathrm{max}_j|\frac{\epsilon_j(f_j(\mathbf{X}_i)-f_0(\mathbf{X}_i))}{\gamma_j} | $ and $U_2 = \mathbb{E}^{1/2}[(\tilde{f}(\mathbf{X})-f_0(\mathbf{X}))^2|\{(\mathbf{X}_i,Y_i)\}_{i=1}^{\infty}]$, we have that

\begin{equation}\label{eq5}
\begin{split}
\mathbb{E}\bigg| \sum_{i=1}^n\epsilon_i(f_{j^{*}}(\mathbf{X}_i)-f_0(\mathbf{X}_i)) \bigg|
&= 
\mathbb{E}\bigg| \frac{\sum_{i=1}^n\epsilon_i(f_{j^{*}}(\mathbf{X}_i)-f_0(\mathbf{X}_i))}{\gamma^{*}}\cdot\gamma^{*} \bigg|\\
&\leq
\mathbb{E}[|T_2(\sqrt{n^{-1}\mathrm{log}N_n}+ \mathbb{E}^{1/2}[(\tilde{f}(\mathbf{X})-f_0(\mathbf{X}))^2|\{(\mathbf{X}_i,Y_i)\}_{i=1}^{\infty}] +\delta )|]\\
&\leq
\mathbb{E}[T_2](\delta+\sqrt{n^{-1}\mathrm{log}N_n})+ \mathbb{E}^{1/2}[T_2^2]\mathbb{E}^{1/2}[(\tilde{f}(\mathbf{X})-f_0(\mathbf{X}))^2].
\end{split}
\end{equation}

Notice that $\mathbb{E}^{1/2}[(\tilde{f}(\mathbf{X})-f_0(\mathbf{X}))^2] = R^{1/2}(\widehat{f},f_0)$. Now, using eq(\ref{eq4}), eq(\ref{eq5}),

\begin{equation}\label{eq6}
\begin{split}
\mathbb{E}\bigg| \frac{2}{n}\sum_{i=1}^n\epsilon_i\tilde{f}(\mathbf{X_i}) \bigg|
\leq
2\delta  + \frac{2}{n}\mathbb{E}[T_2](\delta+\sqrt{n^{-1}\mathrm{log}N_n}) + \frac{2}{n}\mathbb{E}^{1/2}[T_2^2]R^{1/2}(\tilde{f},f_0).
\end{split}
\end{equation}

Let $Z_{ij}:= \frac{f_j(\mathbf{X}_i)-f_0(\mathbf{X}_i)}{\gamma_{j}}$. From $\gamma_j\geq \sqrt{n^{-1}\mathrm{log}N_n}$ and $\gamma_j\geq \mathbb{E}^{1/2}[(f_j(\mathbf{X})-f(\mathbf{X}))^2]$, we know that (1) $|Z_{ij}|\leq 2F/\sqrt{n^{-1}\mathrm{log}N_n}$; (2) $\mathbb{E}Z_{ij}^2\leq 1 $. Observe that $\{\epsilon_i Z_{ij}\}_{i\geq 1}$ is a martingale difference sequence with respect to the filtration $\mathcal{F}_t = \sigma((\mathbf{X}_i),i\leq t) $. Then we check the moment conditions in theorem 1.2B in \cite{de36general}. Since $\epsilon_i$ and $\mathcal{F}_{i}$ are independent, $\mathbb{E}(\epsilon_i^2Z_{ij}^2|\mathcal{F}_{i})=Z_{ij}^2\sigma^2 $. For the $m$-th moment of $|\epsilon_iZ_{ij}|$ given $\mathcal{F}_{i}$,  we have that

\begin{equation*}
\begin{split}
\mathbb{E}[|\epsilon_i^mZ_{ij}^m|\big|\mathcal{F}_{i}]
\leq
\sigma^2 m!c_0^{m-2}|Z_{ij}|^m
\leq
\sigma^2Z_{ij}^2 m!\left(\frac{2c_0F}{\sqrt{n^{-1}\mathrm{log}N_n}}\right)^{m-2}.
\end{split}
\end{equation*}

We split $\mathbb{P}(|\sum_{i=1}^{n}\epsilon_iZ_{ij} |\geq x)$ into two parts:

\begin{equation}\label{eq7}
\begin{split}
\mathbb{P}(|\sum_{i=1}^{n}\epsilon_iZ_{ij} |\geq x)
\leq
\underbrace{\mathbb{P}(|\sum_{i=1}^{n}\epsilon_iZ_{ij} |\geq x, \sigma^2\sum_{i=1}^{n}Z_{ij}^2 \leq \sqrt{\frac{n}{\mathrm{log}N_n}}x)}_{a(x)} +\\ \underbrace{\mathbb{P}(\sum_{i=1}^{n}Z_{ij}^2 \geq \frac{1}{\sigma^2}\sqrt{\frac{n}{\mathrm{log}N_n}}x).}_{b(x)}
\end{split}
\end{equation}

Using theorem 1.2B in \cite{de36general} with $c=2c_0F/\sqrt{(n^{-1})\mathrm{log}N_n}$, $V_n = \sigma^2 \sum_{i=1}^nZ_{ij}^2$, $y=\sqrt{n^{-1}\mathrm{log}N_n}x$, we obtain 

\begin{equation}\label{eq8}
\begin{split}
a(x)
\leq
\mathrm{exp}\left(-\frac{x}{(2+2c_0F)\sqrt{\frac{n}{\mathrm{log}N_n}}}\right).
\end{split}
\end{equation}

To simplify the notation, we let $v_j^2=\mathrm{Var}(Z_{ij}^2)$ and $\tilde{v}_j^2 = \mathrm{Var}\left(Z_{ij}^2\right) + 2\sum_{k>i}\mathrm{Cov}\left( Z_{ij}^2, Z_{kj}^2\right)$. We know that $ (f_j(\mathbf{X}_i)-f_0(\mathbf{X}_i))^2\leq 4F^2$ and $1/\gamma_j^2\leq n$. From lemma \ref{lemma2}, we have

\begin{equation*}
    \begin{split}
\tilde{v}_j^2\leq
([24\mathrm{log}(n)/c]+3)v_j^2+22(16F^4)\frac{1}{n^2(\mathrm{exp}(c/3)-1)}.
    \end{split}
\end{equation*}

Since $v_j^2 \leq \mathbb{E}Z_{ij}^4\leq \mathbb{E}(f_j(\mathbf{X}_i)-f_0(\mathbf{X}_i))^4/(n^{-1}\mathrm{log}N_n\mathbb{E}(f_j(\mathbf{X}_i)-f_0(\mathbf{X}_i))^2)\leq4F^2n/\mathrm{log}N_n$, 
we derive that $\tilde{v}_j^2\leq C_1F^2n\mathrm{log}(n)/\mathrm{log}N_n + C_2F^4/n^2$, where $C_1$ and $C_2$ only depend on $c$.
Since $\mathbb{E}Z_{ij}^2\leq 1$, we have $b(x)\leq \mathbb{P}(\sum_{i=1}^n(Z_{ij}^2-\mathbb{E}Z_{ij}^2)\geq \frac{1}{\sigma^2}\sqrt{\frac{n}{\mathrm{log}N_n}}x-n)$. Observe that $|Z_{ij}^2-\mathbb{E}Z_{ij}^2|\leq 4F^2n/\mathrm{log}N_n$ and $\{Z_{ij}^2\}_i$ is an exponentially $\alpha$-mixing process, using Bernstein inequality (Theorem 2, \cite{merlevede2009bernstein}), again we have that

\begin{equation}\label{eq9}
\begin{split}
b(x)
&\leq
\mathrm{exp}\left( -\frac{C_3(\frac{1}{\sigma^2}\sqrt{\frac{n}{\mathrm{log}N_n}}x-n)^2}{n\tilde{v}_j^2+\frac{16n^2F^4}{\mathrm{log}N_n}+(\frac{1}{\sigma^2}\sqrt{\frac{n}{\mathrm{log}N_n}}x-n)\frac{n\mathrm{log}^2n}{\mathrm{log}N_n}} \right)\\
&\leq
\mathrm{exp}\left( -\frac{C_3(\frac{1}{\sigma^2}\sqrt{\frac{n}{\mathrm{log}N_n}}x-n)^2}{\frac{C_1F^2n^2\mathrm{log}n}{\mathrm{log}N_n} +\frac{C_2F^4}{n} +\frac{16n^2F^4}{\mathrm{log}N_n}+(\frac{1}{\sigma^2}\sqrt{\frac{n}{\mathrm{log}N_n}}x-n)\frac{n\mathrm{log}^2n}{\mathrm{log}N_n}} \right)\\
&\leq
\mathrm{exp}\left( -\frac{C_3(\frac{1}{\sigma^2}\sqrt{\frac{n}{\mathrm{log}N_n}}x-n)^2}{\frac{(C_1F^2+C_2F^4+16F^4)n^2\mathrm{log}n}{\mathrm{log}N_n}+(\frac{1}{\sigma^2}\sqrt{\frac{n}{\mathrm{log}N_n}}x-n)\frac{n\mathrm{log}^2n}{\mathrm{log}N_n}} \right).
\end{split}
\end{equation}

The last inequality uses the assumption that $n\geq \mathrm{log}N_n$.

Next, we use eq(\ref{eq7}), eq(\ref{eq8}), eq(\ref{eq9}) to estimate the upper bound for $\mathbb{E}T_2$ and $\mathbb{E}T_2^2$ in eq(\ref{eq6}). Taking the union bound over all $j=1,\cdots,N_n$, we have

\begin{equation*}
\begin{split}
\mathbb{P}(T_2\geq x) \leq 1 \wedge 2N_n (a(x)+b(x)).
\end{split}
\end{equation*}

Therefore we find that for all $\theta$

\begin{equation}\label{eq10}
\begin{split}
\mathbb{E}T_2\leq \theta\sqrt{n\mathrm{log}N_n} + 2N_n\int_{\theta\sqrt{n\mathrm{log}N_n}}^{\infty}a(x)dx+
2N_n\int_{\theta\sqrt{n\mathrm{log}N_n}}^{\infty}b(x)dx.
\end{split}
\end{equation}

From eq(\ref{eq8}), we have
\begin{equation}\label{eq11}
\begin{split}
2N_n\int_{\theta\sqrt{n\mathrm{log}N_n}}^{\infty}a(x)dx
\leq
2N_n \left((2+2c_0F)\sqrt{\frac{n}{\mathrm{log}N_n}}\right)\mathrm{exp}\left(-\frac{\theta\mathrm{log}N_n}{2+2c_0F}\right).
\end{split}
\end{equation}

When $\theta \geq 2+2c_0F$, it follows that

\begin{equation*}
\begin{split}
2N_n\int_{\theta\sqrt{n\mathrm{log}N_n}}^{\infty}a(x)dx
\leq
2 \left((2+2c_0F)\sqrt{\frac{n}{\mathrm{log}N_n}}\right).
\end{split}
\end{equation*}

From eq(\ref{eq9}), we have that for $\theta>\sigma^2$

\begin{equation*}
\begin{split}
&\int_{\theta\sqrt{n\mathrm{log}N_n}}^{\infty}b(x)dx\\
&\leq
\int_{\theta\sqrt{n\mathrm{log}N_n}}^{\infty} \mathrm{exp}\left( -\frac{C_3(\frac{1}{\sigma^2}\sqrt{\frac{n}{\mathrm{log}N_n}}x-n)^2}{\frac{(C_1F^2+C_2F^4+16F^4)n^2\mathrm{log}n}{\mathrm{log}N_n}+(\frac{1}{\sigma^2}\sqrt{\frac{n}{\mathrm{log}N_n}}x-n)\frac{n\mathrm{log}^2n}{\mathrm{log}N_n}} \right)dx\\
&\leq
\int_{\theta\sqrt{n\mathrm{log}N_n}}^{\infty}
\mathrm{exp}\left( -\frac{C_3(\frac{1}{\sigma^2}\sqrt{\frac{n}{\mathrm{log}N_n}}x-n)}{\frac{(C_1F^2+C_2F^4+16F^4)n\mathrm{log}n}{(\theta/\sigma^2-1)\mathrm{log}N_n}+\frac{n\mathrm{log}^2n}{\mathrm{log}N_n}} \right)dx\\
&\leq
\left(\frac{(C_1F^2+C_2F^4+16F^4)}{(\theta/\sigma^2-1)}+1\right)\frac{\sigma^2\sqrt{n}\mathrm{log}^2n}{C_3\sqrt{\mathrm{log}N_n}}\mathrm{exp}\left( -\frac{C_3(\theta/\sigma^2-1)\mathrm{log}N_n}{\frac{C_1F^2+C_2F^4+16F^4}{\theta/\sigma^2-1}\mathrm{log}n+\mathrm{log}^2n} \right).
\end{split}
\end{equation*}

When $\theta\geq\sigma^2(1+2\mathrm{log}^2n/C_3\vee \sqrt{(2C_1F^2+2C_2F^4+32F^4)\mathrm{log}n/C_3})$,

\begin{equation*}
\begin{split}
\mathrm{exp}\left( -\frac{C_3(\theta/\sigma^2-1)\mathrm{log}N_n}{\frac{C_1F^2+C_2F^4+16F^4}{\theta/\sigma^2-1}\mathrm{log}n+\mathrm{log}^2n} \right).
\leq 
\frac{1}{N_n}
\end{split}
\end{equation*}

Therefore,

\begin{equation}\label{eq12}
\begin{split}
2N_n\int_{\theta\sqrt{n\mathrm{log}N_n}}^{\infty}b(x)dx
\leq
(C_1F^2+C_2F^4+16F^4)\sigma^2\sqrt{\frac{n}{\mathrm{log}N_n}}
+
\frac{2\sigma^2\mathrm{log}^2n}{C_3}\sqrt{\frac{n}{\mathrm{log}N_n}}.
\end{split}
\end{equation}

We choose $\theta = (2+2c_0F)\vee\sigma^2(1+2\mathrm{log}^2n/C_3\vee \sqrt{(2C_1F^2+2C_2F^4+32F^4)\mathrm{log}n/C_3})$. Combining eq(\ref{eq10}), eq(\ref{eq11}) and eq(\ref{eq12}) gives
\begin{equation}\label{eq13}
    \mathbb{E}T \leq C(F,\sigma^2,c)\sqrt{n\mathrm{log}N_n}\mathrm{log}^2n,
\end{equation}
where $C(F,\sigma^2,c)$ is a constant depending on $F, \sigma^2, c$.

Similar to eq(\ref{eq10}), we can prove that

\begin{equation}\label{eq14}
\begin{split}
\mathbb{E}T_2^2\leq \theta^2n\mathrm{log}N_n + 2N_n\int_{\theta^2n\mathrm{log}N_n}^{\infty}a(\sqrt{x})dx+
2N_n\int_{\theta^2n\mathrm{log}N_n }^{\infty}b(\sqrt{x})dx.
\end{split}
\end{equation}

We still choose $\theta = (2+2c_0F)\vee\sigma^2(1+2\mathrm{log}^2n/C_3\vee \sqrt{(2C_1F^2+2C_2F^4+32F^4)\mathrm{log}n/C_3})$. Using the fact that $\int_{b^2}^{\infty} \mathrm{exp}(-(\sqrt{u}-c)a)du = 2(b/a+1/a^2)\mathrm{exp}(-a(b-c))$, we can estimate the upper for both $\int_{\theta^2n\mathrm{log}N_n}^{\infty}a(\sqrt{x})dx $ and $ \int_{\theta^2n\mathrm{log}N_n}^{\infty}b(\sqrt{x})dx$. From eq(\ref{eq8}), we have

\begin{equation*}
\begin{split}
\int_{\theta^2n\mathrm{log}N_n}^{\infty}a(\sqrt{x})dx
&\leq
\int_{\theta^2n\mathrm{log}N_n}^{\infty}\mathrm{exp}\left(-\frac{\sqrt{x}}{(2+2c_0F)\sqrt{\frac{n}{\mathrm{log}N_n}}}\right)dx\\
&=
2\left(\theta n(2+2c_0F)+ (2+2c_0F)^2\frac{n}{\mathrm{log}N_n}\right)\mathrm{exp}\left(-\frac{\theta \sqrt{n\mathrm{log}N_n}}{(2+2c_0F)\sqrt{n/\mathrm{log}N_n}}\right)\\
&\leq
C(F,c)\frac{n\mathrm{log}^2n}{N_n}.
\end{split}
\end{equation*}

From eq(\ref{eq9}), we have

\begin{equation*}
\begin{split}
\int_{\theta^2n\mathrm{log}N_n}^{\infty}b(\sqrt{x})dx
&\leq
\int_{\theta^2n\mathrm{log}N_n}^{\infty}\mathrm{exp}\left( -\frac{C_3(\frac{1}{\sigma^2}\sqrt{\frac{n}{\mathrm{log}N_n}}\sqrt{x}-n)^2}{\frac{(C_1F^2+C_2F^4+16F^4)n^2\mathrm{log}n}{\mathrm{log}N_n}+(\frac{1}{\sigma^2}\sqrt{\frac{n}{\mathrm{log}N_n}}\sqrt{x}-n)\frac{n\mathrm{log}^2n}{\mathrm{log}N_n}} \right)dx\\
&\leq
\int_{\theta^2n\mathrm{log}N_n}^{\infty}
\mathrm{exp}\left( -\frac{C_3(\frac{1}{\sigma^2}\sqrt{\frac{n}{\mathrm{log}N_n}}\sqrt{x}-n)}{\frac{(C_1F^2+C_2F^4+16F^4)n\mathrm{log}n}{(\theta/\sigma^2-1)\mathrm{log}N_n}+\frac{n\mathrm{log}^2n}{\mathrm{log}N_n}} \right)dx\\
&\leq
2\left[ \theta(\frac{C_1F^2+C_2F^4+16F^4}{\theta/\sigma^2-1}+1)n\mathrm{log}^2n + (\frac{C_1F^2+C_2F^4+16F^4}{\theta/\sigma^2-1}+1)^2\frac{n\mathrm{log}^4n}{\mathrm{log}N_n} \right]\\
&\quad \quad \quad \quad
\cdot\mathrm{exp}\left( -\frac{C_3(\theta/\sigma^2-1)\mathrm{log}N_n}{\frac{C_1F^2+C_2F^4+16F^4}{\theta/\sigma^2-1}\mathrm{log}n+\mathrm{log}^2n}\right)\\
&\leq
C(F,\sigma^2,c)n\mathrm{log}^4n/N_n.
\end{split}
\end{equation*}

Therefore, eq(\ref{eq14}) gives

\begin{equation*}
\begin{split}
\mathbb{E}T_2^2 
\leq
C(F,\sigma^2,c)n\mathrm{log}^4n\mathrm{log}N_n.
\end{split}
\end{equation*}

With eq(\ref{eq6}) and upper bound for $\mathbb{E}T_2, \mathbb{E}T_2^2$,

\begin{equation}\label{eq15}
\begin{split}
\mathbb{E}\bigg| \frac{2}{n}\sum_{i=1}^n \epsilon_i \tilde{f}(\mathbf{X})\bigg|
&\leq
2\delta+
\frac{2}{n}C(F,\sigma^2,c)\sqrt{n
\mathrm{log}N_n}\mathrm{log}^2n(\delta+\sqrt{n^{-1}\mathrm{log}N_n}) \\
&\quad\quad\quad
+\frac{2}{n}\sqrt{C(F,\sigma^2,c)n\mathrm{log}^4n\mathrm{log}N_n}R^{1/2}(\tilde{f},f_0)\\
&\leq
\delta C(F,\sigma^2,c)\mathrm{log}^2n + C(F,\sigma^2,c)\ \frac{\mathrm{log}N_n}{n}\mathrm{log}^2n\\
&\quad\quad\quad
+\frac{2}{n}\sqrt{C(F,\sigma^2,c)n\mathrm{log}^4n\mathrm{log}N_n}R^{1/2}(\tilde{f},f_0).
\end{split}
\end{equation}

From the upper bound of inequality in (\uppercase\expandafter{\romannumeral1}), setting $\epsilon=1$ gives

\begin{equation*}
\begin{split}
\sqrt{R(\tilde{f},f_0)}
&\leq
\sqrt{2\widehat{R}_n(\tilde{f},f_0)}
+
\sqrt{2C_F\frac{\mathrm{log}^2n\mathrm{log}N_n}{n}}
+
\sqrt{2C_F\delta\mathrm{log}^2n}\\
&\quad\quad\quad\quad\quad\quad\quad\quad\quad\quad\quad
+
\sqrt{\frac{4C_FF^2\mathrm{log}^4n\mathrm{log}N_n}{n}}\\
&\leq
\sqrt{2\widehat{R}_n(\tilde{f},f_0)}
+
\sqrt{C_F\frac{\mathrm{log}^4n\mathrm{log}N_n}{n}}
+
\sqrt{C_F\delta\mathrm{log}^2n}.
\end{split}
\end{equation*}

Therefore,
\begin{equation}\label{eq16}
\begin{split}
\frac{2}{n}\sqrt{C(F,\sigma^2,c)n\mathrm{log}^4n\mathrm{log}N_n}R^{1/2}(\tilde{f},f_0)
&\leq
\sqrt{\frac{C(F,\sigma^2,c)\mathrm{log}^4n\mathrm{log}N_n}{n}}\widehat{R}_n^{1/2}(\tilde{f},f_0)\\
&\quad
+
C(F,\sigma^2,c)\frac{\mathrm{log}^4n\mathrm{log}N_n}{n} + 2\sqrt{\frac{\delta C_F\mathrm{log}^6n\mathrm{log}N_n}{n}}.
\end{split}
\end{equation}

Using the fact that $2\sqrt{\delta C_F\mathrm{log}^6n\mathrm{log}N_n/n}\leq \delta\mathrm{log}^2n +C_F \mathrm{log}^4n\mathrm{log}N_n/n$, then inequality (\uppercase\expandafter{\romannumeral2}) follows from eq(\ref{eq15}), eq(\ref{eq16}).\\

(\uppercase\expandafter{\romannumeral3}): For any fixed $f\in \mathcal{F}$, $\mathbb{E}[\frac{1}{n}\sum_{i=1}^n(Y_i-\widehat{f}(\mathbf{X}_i))^2]\leq \mathbb{E}[\frac{1}{n}\sum_{i=1}^n(Y_i-f(\mathbf{X}_i))^2]+\Delta_n$. Because of $\mathbf{X}_i\overset{\mathcal{D}}{=}\mathbf{X}$ and $f$ being deterministic, we have $\mathbb{E}[\Vert f-f_0 \Vert_n^2] = \mathbb{E}(f(\mathbf{X})-f_0(\mathbf{X}))^2$. Since also $\mathbb{E}[\epsilon_if(\mathbf{X}_i)]=0$,

\begin{equation}\label{eq17}
\begin{split}
\widehat{R}_n(\widehat{f},f_0)
&\leq
\mathbb{E}[\Vert f-f_0 \Vert_n^2]+
\mathbb{E}[\frac{2}{n}\sum_{i=1}^n\epsilon_i\widehat{f}(\mathbf{X}_i)]+\Delta_n\\
&\leq
\mathbb{E}[\Vert f-f_0 \Vert_n^2]+\Delta_n+
\delta C(F,\sigma^2,c)\mathrm{log}^2n + C(F,\sigma^2,c)\mathrm{log}^4n\frac{\mathrm{log}N_n}{n}\\
&\quad\quad\quad
+2C(F,\sigma^2,c)\sqrt{\frac{\mathrm{log}^4n\mathrm{log}N_n}{n}}\widehat{R}^{1/2}(\widehat{f},f_0).
\end{split}
\end{equation}

Observe that

\begin{equation}\label{eq18}
\begin{split}
2C(F,\sigma^2,c)\sqrt{\frac{\mathrm{log}^4n\mathrm{log}N_n}{n}}\widehat{R}^{1/2}(\widehat{f},f_0)
&\leq
\frac{1+\epsilon}{\epsilon}C^2(F,\sigma^2,c)\mathrm{log}^4n\mathrm{log}N_n/n \\
&\quad\quad\quad
+\frac{\epsilon}{1+\epsilon}\widehat{R}(\widehat{f},f_0).
\end{split}
\end{equation}

Combining eq(\ref{eq17}) with eq(\ref{eq18}) and rearranging $\widehat{R}(\widehat{f},f_0)$ to one side give inequality (\uppercase\expandafter{\romannumeral3}).\\

(\uppercase\expandafter{\romannumeral4}): Let $\tilde{f}\in \mathrm{argmin}_{f\in\mathcal{F}} \sum_{i=1}^{n}(Y_i-f(\mathbf{X}_i))^2$ be an empirical risk minimizer. We have

\begin{equation}\label{eq19}
\begin{split}
&\widehat{R}_n(\widehat{f},f_0) -\widehat{R}_n(\tilde{f},f_0)\\
&=
\Delta_n 
+ 
2\mathbb{E}\left[\frac{1}{n}\sum_{i=1}^n\epsilon_i\widehat{f}(\mathbf{X}_i)\right]
-
2\mathbb{E}\left[\frac{1}{n}\sum_{i=1}^n\epsilon_i\tilde{f}(\mathbf{X}_i)\right]\\
&\geq
\Delta_n 
-
2\delta C(F,\sigma^2,c)\mathrm{log}^2n 
-
2C(F,\sigma^2,c)\mathrm{log}^4n\frac{\mathrm{log}N_n}{n}\\
&-
2C(F,\sigma^2,c)\sqrt{\frac{\mathrm{log}^4n\mathrm{log}N_n}{n}}\widehat{R}_n^{1/2}(\widehat{f},f_0)
-
2C(F,\sigma^2,c)\sqrt{\frac{\mathrm{log}^4n\mathrm{log}N_n}{n}}\widehat{R}_n^{1/2}(\tilde{f},f_0).
\end{split}
\end{equation}

Observe that

\begin{equation}\label{eq20}
\begin{split}
&2C(F,\sigma^2,c)\sqrt{\frac{\mathrm{log}^4n\mathrm{log}N_n}{n}}\widehat{R}^{1/2}(\widehat{f},f_0)
\leq
C^2(F,\sigma^2,c)\frac{\mathrm{log}^4n\mathrm{log}N_n}{n}
+
\widehat{R}(\widehat{f},f_0)\\
&2C(F,\sigma^2,c)\sqrt{\frac{\mathrm{log}^4n\mathrm{log}N_n}{n}}\widehat{R}_n^{1/2}(\tilde{f},f_0)
\leq
\frac{1-\epsilon}{\epsilon}C^2(F,\sigma^2,c)\frac{\mathrm{log}^4n\mathrm{log}N_n}{n}
+
\frac{\epsilon}{1-\epsilon}\widehat{R}_n(\tilde{f},f_0).
\end{split}
\end{equation}

Inequality (\uppercase\expandafter{\romannumeral4}) follows from eq(\ref{eq19}), eq(\ref{eq20}).
\end{proof}

\begin{lemma}\label{lemma4}
Consider the d-variate nonparametric regression model $Y_i=f_0(\mathbf{X}_i)+\epsilon_i$
with unknown regression function $f_0$, satisfying  $\Vert f_0\Vert_{\infty} \leq F$ for some $F\geq 1$. Let $\widehat{f}_n$ be any estimator taking values in the class $\mathcal{F}(L,p,s,F)$ and $\Delta_n$ is defined in lemma \ref{lemma3}. Under Assumptions 1-3, for any $\epsilon \in (0,1]$, there exists a constant $C_{\epsilon}$, only depending on $\epsilon$, such that with

\begin{equation*}
\gamma_{\epsilon,n} := C_{\epsilon}F^2\frac{(s+1)\mathrm{log}(n(s+1)^L)p_0p_{L+1}\mathrm{log}^4n}{n},
\end{equation*}
\begin{equation*}
\begin{split}
&(1-\epsilon)^2\Delta_n(\widehat{f}_n,f_0) -\gamma_{\epsilon,n}\\
&\leq R(\widehat{f}_n,f_0)
\leq (1+\epsilon)^2\left( \mathop{\mathrm{inf}}_{f\in\mathcal{F}(L,p,s,F)}\Vert f-f_0 \Vert_{\infty}^2 + \Delta_n(\widehat{f}_n,f_0) \right)+\gamma_{\epsilon,n}
\end{split}
\end{equation*}
\end{lemma}
\begin{proof}
Lemma \ref{lemma4} follows from Lemma \ref{lemma3} with the choice of $\delta =1/n$, $\mathcal{F}=\mathcal{F}(L,p,s,\infty)$ and Remark 1 in \cite{schmidt2020nonparametric}
\begin{equation*}
\begin{split}
\mathrm{log}\mathcal{N}\left( \delta,\mathcal{F}(L,p,s,\infty),\Vert \cdot \Vert_{\infty} \right) \leq (s+1)\mathrm{log}\left( 2^{2L+5}\delta^{-1}(L+1)p_0^2p_{L+1}^2s^{2L} \right)
\end{split}
\end{equation*}
\end{proof}

\subsection{Proof of Theorem \ref{th1}}\label{supp:sec:sim:1.2}
\begin{proof}[Proof of theorem \ref{th1}]

Combining lemma \ref{lemma4} with the assumed bounds on $L$ and $s$, it follows that

\begin{equation}\label{th1_eq3}
\begin{split}
&\frac{1}{4}\Delta_n(\widehat{f}_n,f_0) - C^{'}\phi_n L\mathrm{log}^6n\leq R(\widehat{f},f_0)\\
&4 \mathop{\mathrm{inf}}_{f^{*}\in \mathcal{F}(L,p,s,F)}\Vert f^{*}-f_0 \Vert_{\infty}^2+4\Delta_n(\widehat{f}_n,f_0) +C^{'}\phi_nL\mathrm{log}^6n,
\end{split}
\end{equation}
where we used $\epsilon=1/2$ for the lower bound and $\epsilon = 1$ for the upper bound. Let $C=8C^{'}$, then $ C^{'}\phi_nL\mathrm{log}^6n \leq\widehat{R}(\widehat{f},f_0)$ whenever $ \Delta_n(\widehat{f},f_0)\geq C\phi_n L\mathrm{log}^6n$. The lower bound on \eqref{th1_eq2} is proved.\\

To get the upper bound, we need to control $\mathop{\mathrm{inf}}_{f^{*}\in \mathcal{F}(L,p,s,F)}\Vert f^{*}-f_0 \Vert_{\infty}^2 $. From eq(26) in \cite{schmidt2020nonparametric}, we know that
\begin{equation*}
\begin{split}
\mathop{\mathrm{inf}}_{f^{*}\in\mathcal{F}(L,p,s)}\Vert f^{*}-f_0 \Vert_{\infty}^2
\leq
C^{'}\mathop{\mathrm{max}}_{i=0,\cdots,q} N^{-\frac{2\beta_i^{*}}{t_i}}
\leq
C^{'}\mathop{\mathrm{max}}_{i=0,\cdots,q} C_0^{-\frac{2\beta_i^{*}}{t_i}} n ^{-\frac{\beta_i^{*}}{2\beta_i^{*}+t_i}}.
\end{split}
\end{equation*}

Hence, $\mathop{\mathrm{inf}}_{f^{*}\in\mathcal{F}(L,p,s,F)}\Vert f^{*}-f_0 \Vert_{\infty}^2\leq C_1 \phi_n $. Therefore, there exists $\tilde{f}\in \mathcal{F}(L,p,s,F)$ such that $\Vert f^{*}-f_0 \Vert_{\infty}^2\leq C_1 \phi_n $. Define $f^{*}= \tilde{f}(\Vert f_0 \Vert_{\infty}/\Vert \tilde{f} \Vert_{\infty} )$. Then $\Vert f^{*} \Vert_{\infty}\leq \Vert f_0 \Vert_{\infty} = \Vert g_q \Vert_{\infty} \leq K\leq F$, which implies that $\tilde{f}\in \mathcal{F}(L,p,s,F)$. Writing $f^{*} - f_0 = (f^{*} - \tilde{f}) + (\tilde{f}_n - f_0)$, we obtain $\Vert f^{*} - f_0 \Vert_{\infty} \leq 2\Vert \tilde{f} - f_0 \Vert_{\infty} \leq 2C_1\phi_0$. From eq(\ref{th1_eq3}), we obtain that
\begin{equation}\label{th1_eq4}
    R(\widehat{f},f_0)\leq
8C_1\phi_0+4\Delta_n(\widehat{f}_n,f_0) +C^{'}\phi_nL\mathrm{log}^6n.
\end{equation}

The upper bounds on eq(\ref{th1_eq1}) and eq(\ref{th1_eq2}) follow from eq(\ref{th1_eq4}).
\end{proof}

\subsection{Proof of Theorem \ref{th2}}\label{supp:sec:sim:1.3}
\begin{proof}[Proof of Theorem \ref{th2}]
From lemma~\ref{lemma3}, we know that
\begin{equation*}
\begin{split}
&(1-\epsilon)^2 \Delta_n - C(F,\sigma^2,c)\frac{\mathrm{log}^4n\mathrm{log}N_n}{n\epsilon}-\delta C(F,\sigma^2,c)\mathrm{log}^2n\\
&\leq R(\widehat{f},f_0)\leq\\
&(1+\epsilon)^2\left(\mathop{\mathrm{inf}}_{f^{*}\in \mathcal{F}} \Vert f^{*}-f_0\Vert_{\infty}^{2}+ \Delta_n(\widehat{f},f)+C(F,\sigma^2,c)\delta \mathrm{log}^2n\right)\\
&+
\frac{(1+\epsilon)^3}{\epsilon}C(F,\sigma^2,c)\frac{\mathrm{log}^4n\mathrm{log}N_n}{n}.
\end{split}
\end{equation*}

Let $
\mathcal{F}:= \mathcal{F}(L,p,s,\infty)$. From remark 1 in \cite{schmidt2020nonparametric},
\begin{equation*}
\begin{split}
\mathrm{log}\mathcal{N}\left( \delta,\mathcal{F}(L,p,s,\infty),\Vert \cdot \Vert_{\infty} \right) \leq (s+1)\mathrm{log}\left( 2^{2L+5}\delta^{-1}(L+1)p_0^2p_{L+1}^2s^{2L} \right).
\end{split}
\end{equation*}

According to the assumption $s\asymp dL$ and $d\lesssim p$ with the choice of $\delta = \frac{1}{n}$, we have that $\mathrm{log}N_n \asymp dL\mathrm{log}n$. Therefore, using the result of lemma \ref{lemma3} with $\epsilon =\frac{1}{2}$, we have that
\begin{equation}\label{th2_eq1}
\begin{split}
&\frac{1}{4} \Delta_n - C^{'}\frac{dL\mathrm{log}^5n}{n}
\leq R(\widehat{f},f_0)\leq\\
& 4\mathop{\mathrm{inf}}_{f^{*}\in \mathcal{F}} \Vert f^{*}-f_0\Vert_{\infty}^{2}+ 4\Delta_n(\widehat{f},f)
+
C^{'}\frac{dL\mathrm{log}^5n}{n}.
\end{split}
\end{equation}

The lower bound of eq(\ref{th2_stat_eq2}) can be derived from the left side of eq(\ref{th2_eq1}) with the assumption $d\asymp n^{\frac{1}{\alpha + 1}}$. To derive the upper bound of eq(\ref{th2_stat_eq1}) and eq(\ref{th2_stat_eq2}), we need the upper bound of $\mathrm{inf}_{f^{*}\in \mathcal{F}(L,\mathbf{p},s,F)} \Vert f^{*}-f_0\Vert_{\infty} $. Let $\tilde{f} = \sum_{i=1}^{d}\phi_i Y_{t-i} \in \mathcal{F}(3, (d,2Kd,1),s), s = 4Kd$ (For each $Y_{t-i}$ in input layer, it maps to $2\lceil \phi_i\rceil$ units in hidden layer. $\lceil \phi_i\rceil$ units are $\phi_i(Y_{t-i})_{+}/\lceil \phi_i\rceil$ and other $\lceil \phi_i\rceil$ units are $-\phi_i(-Y_{t-i})_{+}/\lceil \phi_i\rceil$. Assumption $\phi_i < K$ implies that there are at most $2Kd$ units in hidden layer. The output $\tilde{f}$ equals to the summation of all hidden units). We have that

\begin{equation}\label{th2_eq2}
\begin{split}
\Vert \tilde{f}-f_0\Vert_{\infty}
= \Vert \sum_{i=d+1}^{\infty}\phi_i Y_{t-i} \Vert_{\infty}
\leq K\sum_{d+1}^{\infty}|\phi_i|
\lesssim
\frac{1}{(d+2)^{\alpha}}.
\end{split}
\end{equation}

Define $f^{*}= \tilde{f}(\Vert f_0 \Vert_{\infty}/\Vert \tilde{f} \Vert_{\infty} )$. Then $\Vert f^{*} \Vert_{\infty} \leq \Vert f_0 \Vert_{\infty} \leq K\sum_{i=1}^{\infty}|\phi_i|  \leq KM/2^{\alpha}\leq F $. Therefore, $f^{*}\in \mathcal{F}(4,\mathbf{p}^{*},4Kd+1, F)$. To extend the layer of neural network from $4$ to $L$, we can add $(L-4)$ additional identical layers before $\mathcal{F}(4,\mathbf{p},(6K+2)d, F)$. Note that the dimension of input vector is $d$, the deepened network belongs to 
\[
\mathcal{F}(L,(d,\cdots,d,\mathbf{p}^{*}),(4K+L)d+1, F).
\]

Therefore, the neural network $f^{*} \subset \mathcal{F}(L,\mathbf{p},s, F)$.
Writing $f^{*} - f_0 = (f^{*} - \tilde{f}) + (\tilde{f}_n - f_0)$, we obtain $\Vert f^{*} - f_0 \Vert_{\infty} \leq 2\Vert \tilde{f} - f_0 \Vert_{\infty} $. With eq(\ref{th2_eq1}) and eq(\ref{th2_eq2}), we have that

\begin{equation}
\begin{split}
R(\widehat{f},f_0)\leq
C\frac{1}{d^{\alpha}} + 4\Delta_n(\widehat{f},f)
+
C^{'}\frac{dL\mathrm{log}^5n}{n}.
\end{split}
\end{equation}

Using the assumption $d\asymp n^{\frac{1}{\alpha + 1}}$ again, we can derive the upper bound of eq(\ref{th2_stat_eq1}) and eq(\ref{th2_stat_eq2}).
\end{proof}

\section{Additional details on numerical experiments}\label{supp:sec:sim:2}
\textbf{Computer information}:\\
Processor:   Intel(R) Core(TM) i7-9750H CPU @ 2.60GHz   2.59 GHz\\
Installed RAM: 16.0 GB (15.9 GB usable)\\
System type: Windows 10 Home 64-bit operating system, x64-based processor\\
Disk:           Samsung mzvlb512hbjq-000h1\\
GPU:            NVIDIA GeForce GTX 1660 Ti

% \textbf{Additional figure in Section~\ref{sec:sim:1}}:\\
% \begin{figure}[!ht]
% \centering
% \subfigure[temporally dependent ($\rho = 0.2$)]{
% \begin{minipage}[t]{0.45\linewidth}
% \centering
% \includegraphics[width=2.2in]{pic/simu1_loss_ts_rho2.jpg}
% %\caption{fig1}
% \end{minipage}%
% }%
% \subfigure[independent($\rho = 0.2$)]{
% \begin{minipage}[t]{0.45\linewidth}
% \centering
% \includegraphics[width=2.2in]{pic/simu1_loss_iid_rho2.jpg}
% %\caption{fig2}
% \end{minipage}%
% }%
% \centering
% \caption{logarithm of MSE and $\widehat{R}(\widehat{f},f_0)$ against logarithm of the sample size.}
% \label{fig3}
% \end{figure}

% \newpage
\textbf{Additional figure in Section \ref{sec:real}}: Figure \ref{fig4}.\\

\begin{figure}[!ht]
\centering
\subfigure[$\gamma = 0$]{
\begin{minipage}[t]{0.45\linewidth}
\centering
\includegraphics[width=2.2in]{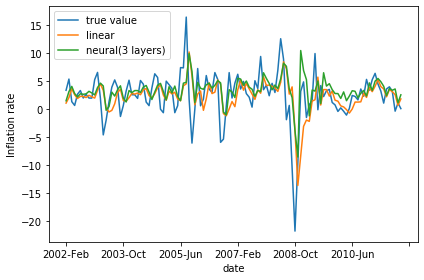}
%\caption{fig1}
\end{minipage}%
}%
\subfigure[$\gamma = 5$]{
\begin{minipage}[t]{0.45\linewidth}
\centering
\includegraphics[width=2.2in]{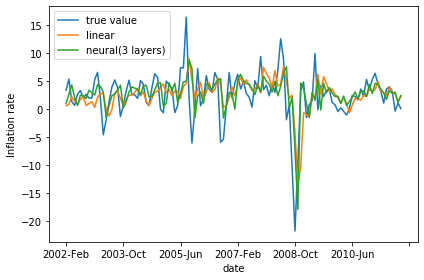}
%\caption{fig2}
\end{minipage}%
}%
\\
\subfigure[$\gamma = 7.5$]{
\begin{minipage}[t]{0.45\linewidth}
\centering
\includegraphics[width=2.2in]{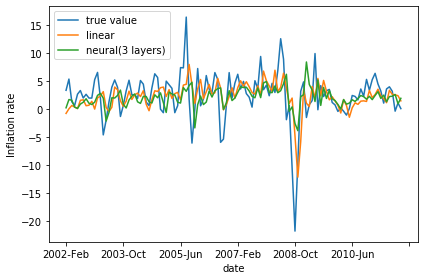}
%\caption{fig1}
\end{minipage}%
}%
\subfigure[$\gamma = 10$]{
\begin{minipage}[t]{0.45\linewidth}
\centering
\includegraphics[width=2.2in]{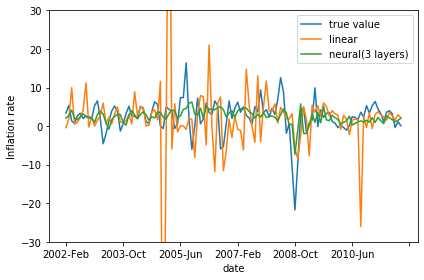}
%\caption{fig2}
\end{minipage}%
}%
\centering
\caption{Estimation for monthly inflation rate. The blue line reflects the true change of inflation rate from 2002 Feb to 2011 Nov. Other two lines correspond to two estimators.}
\label{fig4}
\end{figure}

\section{Numerical comparison between DNN and LSTM}\label{supp:sec:sim:3}
In this section, we compare the predicting performance of DNN with LSTM. We consider the same set of models as in section \ref{sec:sim:2}. The input dimension of both DNN and LSTM are determined by AIC and the number of hidden layer units are set to be 10. The results are shown in figure \ref{fig5}. As we can see from this figure, DNN has a slightly faster convergence rate than LSTM in linear AR models while for non-linear AR models, LSTM performs slightly better. In fact, the difference between the results of DNN and LSTM in these simulation settings are not statistically significant. 

% Analyse the performance of different neural network structures in theoretical way may be the future investigation.

\begin{figure}[!ht]
\centering
\subfigure[model (1)]{
\begin{minipage}[t]{0.45\linewidth}
\centering
\includegraphics[width=2.2in]{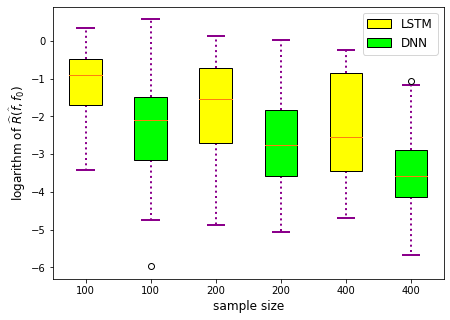}
%\caption{fig1}
\end{minipage}%
}%
\subfigure[model (2)]{
\begin{minipage}[t]{0.45\linewidth}
\centering
\includegraphics[width=2.2in]{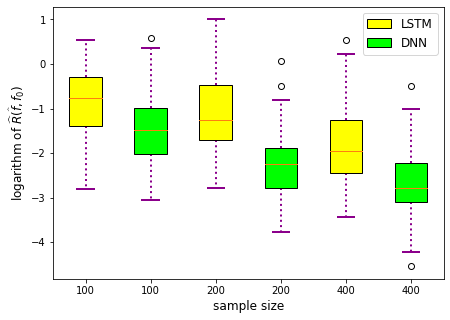}
%\caption{fig2}
\end{minipage}%
}%
\\
\subfigure[model (3)]{
\begin{minipage}[t]{0.45\linewidth}
\centering
\includegraphics[width=2.2in]{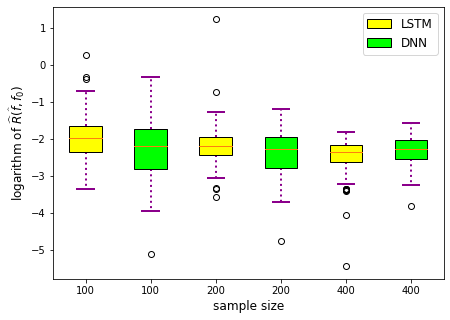}
%\caption{fig1}
\end{minipage}%
}%
\subfigure[model (4)]{
\begin{minipage}[t]{0.45\linewidth}
\centering
\includegraphics[width=2.2in]{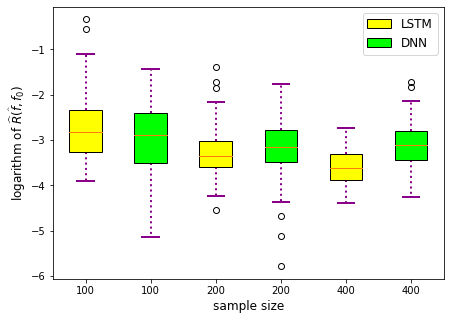}
%\caption{fig2}
\end{minipage}%
}%
\centering
\caption{: Box plots of logarithm of the mean square error on the testing set as a function of sample
size.}
\label{fig5}
\end{figure}

\end{document}